\documentclass{article}
\usepackage[utf8]{inputenc}
\usepackage[displaymath, mathlines]{lineno}
\usepackage{amsmath,amsfonts,amsthm,amssymb}
\usepackage{setspace}
\usepackage{fullpage}
\usepackage{fancyhdr}
\usepackage{xspace}

\usepackage{authblk}

\usepackage{enumerate}
\usepackage{verbatim}
\usepackage{lastpage}
\usepackage{extramarks}
\usepackage{chngpage}
\usepackage{soul,color}
\usepackage{graphicx,float,wrapfig}
\usepackage{ifpdf}
\usepackage{listings}
\usepackage{diagbox}
\usepackage{algorithm}
\usepackage{algorithmicx}
\usepackage{algpseudocode}
\ifpdf
  \usepackage[pdftex,bookmarks,bookmarksopen,bookmarksdepth=3]{hyperref}
\else
  \usepackage[pagebackref]{hyperref}
\fi

\newcommand{\E}{\mathop{\mathbb{E}}}
\newcommand{\R}{\mathbb{R}}

\newcommand{\MF}{\mathcal{F}}
\newcommand{\MG}{\mathcal{G}}

\newcommand{\MA}{\mathcal{A}}
\newcommand{\MX}{\mathcal{X}}
\newcommand{\MY}{\mathcal{Y}}

\newcommand{\MD}{\mathcal{D}}
\newcommand{\ME}{\mathcal{E}}
\newcommand{\MC}{\mathcal{C}}
\newcommand{\MZ}{\mathcal{Z}}
\newcommand{\MN}{\mathcal{N}}
\newcommand{\MV}{\mathcal{V}}
\newcommand{\MH}{\mathcal{H}}

\DeclareMathOperator{\supp}{supp}
\DeclareMathOperator{\vol}{vol}
\DeclareMathOperator{\Var}{Var}

\newcommand{\Main}{\textsc{AVE-Main}\xspace}
\newcommand{\Chk}{\textsc{Check}\xspace}
\newcommand{\Identify}{\textsc{Identify}\xspace}
\newcommand{\Eliminate}{\textsc{Eliminate}\xspace}

\newcommand*{\defeq}{\stackrel{\text{def}}{=}}

\newtheorem{theorem}{Theorem}[section]
\newtheorem{lemma}[theorem]{Lemma}
\newtheorem{assumption}[theorem]{Assumption}

\newtheorem{corollary}[theorem]{Corollary}
\newtheorem{remark}[theorem]{Remark}
\theoremstyle{definition}

\newtheorem{definition}[theorem]{Definition}

\newcommand{\citep}{\cite}

\renewcommand{\tilde}{\widetilde}
\renewcommand{\hat}{\widehat}

\title{$\sqrt{n}$-Regret for Learning in Markov Decision Processes with Function Approximation and Low Bellman Rank}
\begin{document}
\author[1,2]{Kefan Dong\thanks{Accepted for presentation at the Conference on Learning Theory (COLT) 2020. Author names are listed in alphabetical order. Correspondence to: {\tt yuanz@illinois.edu}. Work done while Kefan Dong was a visiting student at UIUC. Kefan Dong and Yuan Zhou were supported by a Ye Grant.}}
\author[2]{Jian Peng}
\author[3]{Yining Wang}
\author[4]{Yuan Zhou}
\affil[1]{\normalsize Institute for Interdisciplinary Information Sciences, Tsinghua University, Beijing, China}
\affil[2]{\normalsize Department of Computer Science, University of Illinois Urbana-Champaign, Urbana, IL 61820, USA}
\affil[3]{\normalsize Warrington College of Business, University of Florida, Gainesville, FL 32611, USA}
\affil[4]{\normalsize Department of ISE, University of Illinois Urbana-Champaign, Urbana, IL 61820, USA}

\maketitle

\begin{abstract}
In this paper, we consider the problem of \emph{online} learning of Markov decision processes (MDPs) with very large state spaces.
Under the assumptions of realizable function approximation and low Bellman ranks, we develop an online learning algorithm that learns the optimal value function
while at the same time achieving very low cumulative \emph{regret} during the learning process.
Our learning algorithm, Adaptive Value-function Elimination (\textsf{AVE}), is inspired by the policy elimination algorithm proposed in \cite{Jiang2017ContextualDP}, known as \textsf{OLIVE}.
One of our key technical contributions in \textsf{AVE} is to formulate the elimination steps in \textsf{OLIVE} as \emph{contextual bandit} problems. This technique enables us to apply the active elimination and expert weighting methods from \cite{dudik2011efficient}, instead of the random action exploration scheme used in the original \textsf{OLIVE} algorithm, for more efficient exploration and better control of the regret incurred in each policy elimination step. To the best of our knowledge, this is the first $\sqrt{n}$-regret result for reinforcement learning in stochastic MDPs with general value function approximation.
\end{abstract}

\section{Introduction}

Consider a Markov Decision Process (MDP) $\mathcal M=(\mathcal X,\mathcal A,H,p,r)$ with state space $\mathcal X$, action space $\mathcal A$, horizon $H$,
transition probabilities $p:\mathcal X\times\mathcal A\to\Delta(\mathcal X)$\footnote{$\Delta(\mathcal X)$ denotes all probability distributions over $\mathcal X$.}
 and reward function $r:\mathcal X\times\mathcal A\to\mathbb R$.
For notational simplicity, we assume that $\mathcal X$ can be partitioned into disjoint subsets as $\mathcal X=\mathcal X_1\cup\cdots\mathcal X_H$, such that $\mathcal X_h\cap\mathcal X_{h'}=\emptyset$ if  $h\neq h'$.
A \emph{policy} $\pi:\mathcal X\to\Delta(\mathcal A)$ is a function that maps a state $x\in\mathcal X$ to a distribution over actions $a\in\mathcal A$.
The objective of policy learning is usually formulated as an optimization of finding $\pi$ that achieves as large the \emph{expected reward} as possible under $\mathcal M$, which is defined as
\begin{equation}
\textstyle
R(\pi) := \mathbb E\big[\sum_{h=1}^H r_h| r_h\sim r(x_h,a_h), a_h\sim\pi(x_h), x_h\sim p(x_{h-1},a_{h-1})\big].
\label{eq:reward-pi}
\end{equation}
The \emph{optimal policy} $\pi$ that maximizes Eq.~(\ref{eq:reward-pi}) is denoted as $\pi^*$.
Without further confusion, for \emph{deterministic} policies (i.e., policies whose $\pi(\cdot)$ is a singleton for all states) 
we abuse the notation $\pi(x)\in\mathcal A$ for the action the policy takes at state $x$.
We remark that the optimal policy $\pi^*$ can always be made deterministic.

When the full specification of the MDP $\mathcal M$ is known, a near-optimal policy $\pi$ can be computed via the Bellman equation and (approximate) dynamic programming,
and is quite well understood in the literature \cite{bertsekas1995dynamic,powell2007approximate}.
In practical scenarios, however, it is usually the case that either the transition probabilities $p$ or the reward function $r$ (or both)
are \emph{unknown}, which need to be estimated, either implicitly or explicitly, through samples or rollout trajectories.
Such learning/planning problems with unknown $p$ and $r$ are also referred to as \emph{reinforcement learning}
and encapsulate several important artificial intelligence applications such as computer games \citep{mnih2015human,bellemare2016unifying}, board games \citep{silver2016mastering,silver2017mastering},
robotic manipulation \citep{gu2017deep}, and many more.

In this paper, we consider the problem of learning near-optimal policy $\pi$ with unknown $p$ and $r$ from two perspectives:
the \emph{sample complexity} perspective, which seeks for the smallest number of realized trajectories (possibly obtained using different exploration policies)
in order to obtain a good policy with high probability, and the \emph{online learning} perspective which characterizes how the exploration policies themselves evolve and improve 
over time.
In the rest of this section, we lay out the basic assumptions and our main results and contributions. 
A more detailed \emph{technical} overview of our results is given in Sec.~\ref{sec:prelim}.

\subsection{Function approximation}

When the state space $\mathcal X$ is finite with small cardinality $|\mathcal X|$, all states $x$ can be enumerated in learning. This is known as the \emph{tabular} MDP setting, which has been extensively studied \citep{jin2018q, zanette2019tighter, azar2017minimax, strehl2006pac, azar2011speedy,even2003learning, sidford2018variance}.
In many real-world problems, however, $|\mathcal X|$ can be very large or even infinite.
For example, in the Go game, the total number of states could be as large as $2\times10^{170}$, clearly infeasible for any approach that attempts to enumerate them.

It is clear that, in order to handle MDPs with very large state spaces, aggressive compression of the state space is required for practical purposes.
In the literature, such compression is most naturally accomplished by the idea of \emph{function approximation}, 
which considers a finite class\footnote{The requirement that $\mathcal F$ is finite could be removed, as shown in Sec.~\ref{sec:infinite} later in this paper.} of functions
$\mathcal F=\{f:\mathcal X\times\mathcal A\to\mathbb R\}$
and restricts ourselves to policies $\Pi=\{\pi_f:f\in\mathcal F\}$ ``induced'' by certain function approximates, defined as
\begin{equation}
\pi_f(x) = \arg\max_{a\in\mathcal A}f(x,a).
\label{eq:pif}
\end{equation}
In essence, the complexity of the function class $\mathcal F$ captures all inherent structures in the MDP $\mathcal M$ with a very large state space.
In practice, the approximation function classes range from linear or low-degree polynomials in revenue management problems \citep{talluri1998analysis,adelman2007dynamic} to very complicated convolutional or recurrent neural networks for complex games \citep{mnih2015human,mnih2016asynchronous}.

To ensure a considered function approximation is appropriate, we impose the following \emph{realizability} assumption 
which guarantees the correspondence between the optimal policy $\pi^*$ and a function $f^*\in\mathcal F$:
\begin{assumption}[Realizability]\label{assumption:realizability}
For the optimal policy $\pi^*$, there exists $f^*\in\mathcal F$ such that for all $h\in[H]$, 
$x_h\in\mathcal X_h$ and $a_h\in\mathcal A$,
\begin{equation}
\textstyle
Q_h^{\pi^*}(x_h,a_h) := \mathbb E\big[\sum_{h'\geq h}r_{h'}\big| r_{h'}=r(x_{h'},a_{h'}), a_{h'}= \pi^*(x_{h'}), s_{h'}\sim p(x_{h'-1},a_{h'-1})\big] = f^*(x_h,a_h).
\label{eq:qh}
\end{equation}
\label{asmp:realizability}
\end{assumption}

We remark that Assumption \ref{asmp:realizability} is a \emph{monotonic} assumption, meaning that if it holds for function class $\mathcal F$
then it also holds for all $\mathcal F'\supseteq\mathcal F$.
While such monotonicity property is desirable, allowing us to use slightly more than necessary function approximators,
such property does \emph{not} hold for many ``completeness'' type conditions in the literature, as we remark in more details in Sec.~\ref{sec:related-work}.

\subsection{Bellman factorization and Bellman rank}
It is easy to verify from definition that the ``Q-function'' defined in Eq.~(\ref{eq:qh}) satisfies the celebrated \emph{Bellman's equation}, which states that for any timestep or layer $h$, state $x_h$ and action $a_h$, 
\begin{equation}
Q_h^{\pi^*}(x_h,a_h) := r(x_h,a_h) + \mathbb E_{x_{h+1}\sim p(\cdot\mid x_h,a_h)}\Big[\max_{a_{h+1}\in\mathcal A} Q_{h+1}^{\pi^*}(x_{h+1},a_{h+1})\Big],
\label{eq:bellman-recursion}
\end{equation}
with the boundary condition that $Q_{H+1}^{\pi^*}\equiv 0$.
With function approximation, Eq.~(\ref{eq:bellman-recursion}) can be simplified as
\begin{equation*}
\Delta(f^*,x_h,a_h) := \mathbb E\big[f^*(x_h,a_h) - r(x_h,a_h) - f^*(x_{h+1},a_{h+1})\big| x_{h+1}\sim p(x_h,a_h), a_{h+1}=\pi^*(x_{h+1})\big] = 0
\end{equation*}
for all $h\in[H], x_h\in\mathcal X_h$ and $a_h\in\mathcal A$.
An ``averaging'' extension of $\Delta(f^*,x_h,a_h)$ for function approximation $f\in\mathcal F$ and ``roll-in'' policy $\pi$ can then be defined as
\begin{equation}
\ME(f,\pi,h):= \E_{x\sim \MD_{\pi, h}}\E_{x'\sim p(\cdot\mid x,\pi_f(x))}\left[f(x,\pi_f(x))-r(x,\pi_f(x))-f(x',\pi_f(x'))\right],
\end{equation}
where $\pi_f$ is the induced policy of $f$ as defined in Eq.~(\ref{eq:pif}),
and $\MD_{\pi,h}$ is the distribution of $x_h$ when policy $\pi$ is performed at timestep $1, 2, \cdots,h-1$.
Eq.~(\ref{eq:bellman-recursion}) then implies that $\ME(f^*,\pi,h)=0$ for all policy $\pi$ and $h$,
and any function $f\in\mathcal F$ satisfying $\ME(f,\pi,h)=0$ for all $\pi$ and $h$ potentially corresponds to the optimal policy.

It turns out that, when using the Bellman error $\ME(f,\pi,h)$ as criteria to eliminate incorrect function approximation and/or sub-optimal policies, 
the behavior of $\ME(f,\pi,h)$ with ``self-induced'' roll-in policies $\pi\in\{\pi_g:g\in\mathcal F\}$ plays important roles in the complexity of the learning problem.
A ``low Bellman rank'' assumption is thus imposed, described in the following:
\begin{assumption}
\label{assumption:bellman_rank} 
There exists $M\ll |\mathcal F|$ such that,
for any $h\in[H]$ and function approximation $f,g\in\mathcal F$, $\ME(f,\pi_g,h)$ can be decomposed as
$$
\ME(f,\pi_g,h) = \langle\nu_h(g),\xi_h(f)\rangle
$$
for some $\nu_h(g),\xi_h(f)\in\mathbb R^M$ satisfying $\|\nu_h(g)\|_2\|\xi_h(f)\|_2\leq\zeta<\infty$.
\end{assumption}

While Assumption \ref{assumption:bellman_rank} is largely a theoretically motivated assumption, 
it also holds in many interesting examples such as tabular MDPs with low-rank transitions, linear quadratic regulators,
POMDPs with small hidden state spaces, gird-world environments (e.g., the Malmo platform \citep{johnson2016malmo}), etc.
Interested readers should refer to the work of \cite{Jiang2017ContextualDP} for detailed motivations and examples of the low-Bellman-rank assumption.

\subsection{From PAC-learning to online learning}

Suppose the learning algorithm has access to $n$ sequentially collected trajectories,
and an \emph{adaptive} policy $\pi^{(i)}$ can be used to generate the $i$th trajectory, which might depend on the algorithm's observations from the previous $(i-1)$ realized trajectories.
Under the ``Probably Approximately Correct (PAC)'' framework, after observing data from $n$ trajectories with $n$ depending \emph{polynomially} on the problem size,
the algorithm is asked to output a policy $\widehat\pi$ which is near-optimal with high probability.
The work in \cite{Jiang2017ContextualDP} provided the first PAC-learning result under Assumptions \ref{asmp:realizability} and \ref{assumption:bellman_rank}:
\begin{theorem}[\cite{Jiang2017ContextualDP}]
There exists an algorithm and a model-dependent constant $C_{\mathcal M}$ that is a polynomial of $H,|\mathcal A|,M,\zeta$ and $\log|\mathcal F|$ such that, 
for any $\varepsilon\in(0,1/2]$, with $n=\widetilde O(C_{\mathcal M}/\varepsilon^2)$ sample trajectories, the algorithm outputs a policy $\hat\pi$ that satisfies
$R(\hat\pi)\geq R(\pi^*)-\varepsilon$ with probability at least 0.9.
\label{thm:pac-learning}
\end{theorem}

While PAC-learning results such as the one in Theorem \ref{thm:pac-learning} is very much desirable, the framework overlooks the 
aspect of exploration policy \emph{improvement}, which expects the quality of the exploration policy to continuously improve as more data are collected.
Such exploration policy improvement is important in applications where bad policies maybe lead to significant loss or even the cost of human lives,
such as learning for self-driving cars.
In these applications, an evaluation criterion of the ``cumulative'' gap of sub-optimality between the committed exploration policies and the 
optimal policy, known commonly as the cumulative \emph{regret} in the online/bandit learning literature,
 is more suitable to measure the quality of policy improvement.

The following theorem is the main result we established in this paper:

\begin{theorem}[Our results, informal]
There exists an algorithm and a model-dependent constant $C_{\mathcal M}'$ that is a polynomial of $H,|\mathcal A|,M,\zeta$, $\log|\mathcal F|$ and $\log (1/\delta)$, such that,
for sufficiently large $n$, the policies $\hat\pi^{(1)},\cdots,\hat\pi^{(n)}$ the algorithm performs on the $n$ trajectories satisfy
with probability $(1-\delta)$ that
$$
\sum_{i=1}^n R(\pi^*) - R(\hat\pi^{(i)}) = \tilde O(C_{\mathcal M}'\times \sqrt{n}).
$$
\label{thm:main-result}
\end{theorem}
\begin{remark}
In addition to Assumptions \ref{asmp:realizability} and \ref{assumption:bellman_rank}, Theorem \ref{thm:main-result} requires several
additional mild assumptions, to be described in Sec.~\ref{sec:prelim}.
\end{remark}

At a higher level, the result of Theorem \ref{thm:main-result} upper bounds the sub-optimalty gap of exploration policies $\{\hat\pi^{(i)}\}$
for \emph{every} trajectory $i=1,2,\cdots,n$ the algorithm obtains.
Because the upper bound is on the order of $\tilde O(\sqrt{n})$, the exploration policies have to constantly improve over themselves
as otherwise a linear $O(n)$ regret will be incurred.

We make some additional remarks on Theorem \ref{thm:main-result}, regarding its connection with the PAC-learning result in Theorem \ref{thm:pac-learning}.
\begin{remark}[online-to-batch conversion]
Because the expected reward function $R(\pi)$ is \emph{linear} in policy $\pi$,
by considering the ``averaging policy'' $\overline\pi = \frac{1}{n}\sum_{i=1}^n\hat\pi^{(i)}$ one has
$R(\overline\pi)= \frac{1}{n}\sum_{i=1}^n[R(\pi^*)-R(\hat\pi^{(i)})] \geq R(\pi^*) - \tilde O(C_{\mathcal M}'/\sqrt{n})$ with high probability,
matching the result in Theorem \ref{thm:pac-learning}.
\end{remark}
\begin{remark}[exploration and exploitation]
By running the PAC-learning algorithm implied by Theorem \ref{thm:pac-learning} on the first $n^{1/3}$ sample trajectories 
and then switching to the learnt policy $\pi$ for the rest of the $n-n^{1/3}$ trajectories,
one obtain a regret upper bound of $\tilde O(\mathcal C_m\times n^{2/3})$, much worse than the $\tilde O(\sqrt{n})$ upper bound
in Theorem \ref{thm:main-result}.
The $\tilde O(n^{2/3})$ regret bound cannot be improved by simply treating the PAC-learning algorithm as a black box.
\end{remark}

\subsection{Notations}

For two sequences $\{a_n\}$ and $\{b_n\}$, we denote $a_n=O(b_n)$ or $a_n\lesssim b_n$ if there exists a \emph{universal} constant $C<\infty$
such that $\limsup_{n\to\infty} |a_n|/|b_n|\leq C$.
Similarly, we denote $a_n=\Omega(b_n)$ or $a_n\gtrsim b_n$ if there exists a \emph{universal} constant $c>0$ such that $\liminf_{n\to\infty} |a_n|/|b_n| \geq c$.
We denote $a_n=\Theta(b_n)$ or $a_n\asymp b_n$ if both $a_n\lesssim b_n$ and $a_n\gtrsim b_n$ hold.

We use $A = |\mathcal{A}|$ to denote the size of the action space. For any value hypothesis function $f$, we use $\pi_f$ to denote the policy that acts greedily according to $f$. Given $f$ and a timestep or a layer $h \in [H]$, we use $\mathcal{D}_{\pi_f, h}$ to denote the distribution of states at layer $h$ when we use policy $\pi_f$. We also abuse the notation and use $\mathcal{D}_{f, h}$ to denote $\mathcal{D}_{\pi_f, h}$. We use $\pi_U$ to denote the uniformly random policy, i.e., the policy always chooses a random action from $\mathcal{A}$ uniformly. 
Given a distribution $G$ of hypothesis functions and a sub-class $\mathcal{G}'$ of functions, we define $G\rvert_{\mathcal{G}'}$ to be the projection of $G$ onto $\mathcal{G}'$. I.e., we let
\[
\left(G\rvert_{\mathcal{G}'}\right)(f) = \begin{cases}\Pr_{f \sim G}[f \in \mathcal{G}']^{-1} \cdot G(f),&\text{when }f\in \mathcal{G}',\\0,&\text{otherwise.}\end{cases} 
\]
Given two hypothesis functions $f_1$ and $f_2$, and a layer $h$, we define $f_1 \circ_h f_2$ to be the concatenation of the two functions at layer $h$. More specifically, we set
\[
(f_1\circ_h f_2)(x,a)=\begin{cases}f_1(x,a),&x\in \bigcup_{h'=1}^{h-1}\MX_{h'},\\ f_2(x,a), &\text{otherwise.}\end{cases} 
\]

\section{Technical overview}\label{sec:prelim}

This section gives a very high-level technical overview of our algorithm and analysis.
We start with an overview of the \textsf{OLIVE} algorithm introduced in \cite{Jiang2017ContextualDP}
attaining the PAC-learning guarantee as described in Theorem \ref{thm:pac-learning}.
We point out two key technical challenges which prevent us from simply transforming the algorithm to achieve good regret results.
We then continue with the description of our high-level ideas for designing \textsf{AVE} to circumvent the two challenges.
Finally, we introduce two additional assumptions, both are very mild compared to the core assumptions of \ref{asmp:realizability}
and \ref{assumption:bellman_rank}.

\subsection{An overview of the \textsf{OLIVE} algorithm}

The \textsf{OLIVE} algorithm proposed by \cite{Jiang2017ContextualDP} is based on the ideas of \emph{optimistic exploration}
and \emph{policy elimination}.
To describe the algorithm, we define the \emph{value function} $V^{\pi}_{h}(x_h)$ of policy $\pi$ at layer $h\in[H]$ and state $x_h\in\mathcal X_h$
as 
\begin{equation}
\textstyle
V^{\pi}_{h}(x_h) := \mathbb E\big[\sum_{h'=h}^H r_{h'}\big| r_{h'}=r_{h'}(x_{h'},a_{h'}), a_{h'}=\pi(x_{h'}), x_{h'}\sim p(x_{h'-1},a_{h'-1})\big]
\label{eq:value-function}
\end{equation}
as the expected reward collected on layers $h,h+1,\cdots,H$ under the roll-out policy $\pi$. 
When it is clear from the context, we omit the subscript $h$  in the notation.

Naturally, the values of $V^{\pi}_{h}$ are inaccessible to the learning algorithm because neither $r$ or $p$ are known.
However, if a function $f\in\mathcal F$ is a reasonable approximation of its corresponding policy $\pi_f$, an estimated value function
$\hat V^{\pi}_{h}$ can be defined as
\begin{equation}
\hat V_{h}^{\pi_f}[x_h] := \max_{a\in\mathcal A} f(x_h, a),
\label{eq:hat-value-function}
\end{equation}
and an estimate on the expected total reward $R(\pi_f)$ can be obtained as $\hat R(\pi_f) = \hat V_{1}^{\pi_f}(x_1)$.

The \textsf{OLIVE} algorithm maintains an ``active'' function class $\mathcal G\subseteq\mathcal F$
initialized as the entire function class $\mathcal G=\mathcal F$.
It proceeds shrinking the active function class $\mathcal G$ as follows:
\begin{enumerate}
\item Find $f\in\mathcal G$ that maximizes $\hat R(\pi_f)=\hat V_{1}^{\pi_f}(x_1)$;
\item Verify whether $|\sum_h\ME(f,\pi_f,h)|$ is close to zero using sample trajectories; 
if $|\sum_h\ME(f,\pi_f,h)|\lesssim\varepsilon$ then terminate the algorithm and output $\pi_f$ as a good policy;
otherwise identify $h\in\mathcal H$ such that $|\ME(f,\pi_f,h)|\gtrsim \varepsilon/H$,
which must exist because $|\sum_h\ME(f,\pi_f,h)|$ is large;
\item Remove all $g\in\mathcal G$ from $G$ with $|\ME(g,\pi_f,h)|\gtrsim\varepsilon$; More specifically,\label{item:olive-elimiate}
\begin{itemize}
\item[3.a.] Collect $n'$ sample trajectories with roll-in policy $\pi_f$ up to layer $h$, \emph{random} action $a_h\in\mathcal A$ in layer $h$,
and arbitrary roll-out policies after layer $h$;
\item[3.b.] $|\ME(g,\pi_f,h)|$ can then be estimated using importance sampling, because for any state $x_h$ there are in expectation $n'/|\mathcal A|$ trajectories committed to $a_h=\pi_g(x_h)$;
\end{itemize}
\item Repeat steps 1 to 3 with the smaller active function class $\mathcal G$, until a policy is produced.
\end{enumerate}

The correctness of the above procedure is based on several simple observations: first, the optimal policy $\pi^*=\pi_{f^*}$ will never be eliminated because $\ME(f^*,\pi,h)=0$ for all roll-in policies $\pi$ and layer $h$;
second, a careful decomposition of value estimation error reveals that $\hat V^{\pi_f}_{1}(x_1)-V^{\pi_f}_{1}(x_1)=\sum_{h=1}^H\ME(f,\pi_f,h)$
holds for all policy $\pi$, and therefore a small total Bellman error $|\sum_h\ME(f,\pi_f,h)|$ together
with the optimistic exploration oracle (i.e., explore $\arg\max_{f\in\mathcal G}\hat V^{\pi_f}_{1}(x_1)$) implies good performance 
whenever the algorithm terminates and produces a policy $\pi_f$;
finally, due to the low-Bellman-rank assumption \ref{assumption:bellman_rank},
it can be shown via a volumetric argument that the size of $\mathcal G$ will decrease rapidly and eventually contains only the 
optimal function $f^*$ and its close neighborhoods after $\mathrm{poly}(M,H)$ iterations of steps 1--3.

\subsection{Key technical challenges and high-level ideas of our analysis}

The first step of transforming the \textsf{OLIVE} algorithm into a regret-aware one is to replace 
Bellman-error estimates with their adaptive counterparts.
More specifically, for any error level $\epsilon\in(0,1)$, it is possible to distinguish with high probability between the two cases of $|\sum_h\ME(f,\pi_f,h)|\gtrsim\epsilon$
and $|\sum_h\ME(f,\pi_f,h)|\ll\epsilon$ (or the $|\ME(g,\pi_f,h)|$ term as well) with $n'\approx 1/\epsilon^2$ sample trajectories,
using either union bounds over all trajectories or the ``doubling trick'' \citep{auer1995gambling}.
With such adaptive Bellman-error estimators, two key technical challenges can then be identified:
\begin{enumerate}
\item[(C1)] Because the optimisitc exploration policy $f\in\arg\max_{f\in\mathcal G}\hat V^{\pi_f}_{1}(x_1)$
might have arbitrarily large total Bellman error (i.e., $|\sum_h \ME(f,\pi_f,h)|\approx \epsilon \gg 1/\sqrt{n}$), to ensure low regret the elimination step afterwards can only be done
on the level of $\epsilon$ (i.e., eliminate all $g\in\mathcal G$ with $|\ME(g,\pi_f,h)|\gtrsim \epsilon$ using $1/\epsilon^2$ samples) instead of the 
``target error level'' $ 1/\sqrt{n}$;
\item[(C2)] When performing policy elimination, the ``random action'' idea in Step 3.a of \textsf{OLIVE} can no longer be used,
as taking a random action at layer $h$ might incur unacceptably large regret.
\end{enumerate}

To overcome challenge (C1), we revise the volumetric argument in \cite{Jiang2017ContextualDP} to 
analyze the progress of volume shrinkage when $g\in\mathcal G$ are only eliminated on the level of $\epsilon\gg 1/\sqrt{n}$,
provided that $|\ME(f,\pi_f,h)|$ itself is as large as $\Omega(\epsilon)$.
We prove that, while exploration policies $f$ with the same Bellman decomposition direction but very different magnitudes
might be visited more than once,
unlike the original \textsf{OLIVE} algorithm in which no direction will be visited twice,
the revised volumetric argument still provides sufficient progress in terms of $\mathcal G$ shrinkage
and essentially upper bounds the number of exploration policies to be a polynomial of $M$ and $H$.

The second challenge (C2), which concerns the inadequacy of the random action exploration in layer $h$,
turns out to be a more fundamental challenge. To understand how our algorithm and analysis overcome this issue,
it is instructive to first consider the simpler case of $h=H$.
As $h$ is now the last layer of each trajectory, the choice of $a_H\in\mathcal A$ will \emph{not}
have any lasting impact beyond the immediate reward $r_H$.
Subsequently, the question of eliminating all $g\in\mathcal G$ with large $|\ME(g,\pi_f,H)|$
can be reduced to a \emph{contextual bandit} problem, with $x_H$ the input context,
$g\in G$ the experts and rewards of action $a\in\mathcal A$ under context $x_H$ being simply $r_H(x_H,a)$.
The active elimination procedure proposed in \citep{dudik2011efficient} is used to solve this contextual bandit problem
on the last layer, which not only delivers low regret but also identifies experts/functions $g\in\mathcal G$
with small regret $\mathbb E_{x_H\sim \mathcal{D}_{\pi_f,H}}[r(x_H,\pi^*(x_H))-r(x_H,\pi_g(x_H))]$.
After a set of functions $g\in\mathcal G$ with small regret is obtained, importance sampling can be carried out to determine their Bellman errors $|\ME(g,\pi_f,H)|$
without incurring large regret.
Note also that, unlike the original \textsf{OLIVE} algorithm, functions $g\in\mathcal G$ with small or even zero Bellman error might get eliminated due to their 
large overall regret.
This shall not cause a problem because the optimal function $f^*$, having both zero Bellman error and zero regret, will never get eliminated.

In intermediate layers $h<H$, the problem becomes more complicated.
The contextual bandit formulation in the previous paragraph cannot be directly applied, as the expected total onward revenue
$V_{h}^{\pi_g}(x_h)$ is not only a function of $a_h\in\mathcal A$ but also the policy $\pi_g$ itself.
Roughly speaking, our algorithm in intermediate layers $h<H$ will use the algorithm in \citep{dudik2011efficient}
to accomplish either of the following objectives:
\begin{enumerate}
\item[(a)] construct a distribution over $g\in\mathcal G$, $G$, such that both the \emph{average} regret
$\mathbb E_{x_h\sim D{\pi_f,h}}\mathbb E_{g\in G}[V^{\pi^*}_{h}(x_h)-V^{\pi_g}_{h}(x_h)]$ 
and the variance of estimating $|\ME(g,\pi_f,h)|$, $g\in \mathcal G$ are small; or
\item[(b)] identify a $g\in\mathcal G$, $h'>h$ and $\epsilon'\in(0,1)$ such that $|\ME(g,\pi_g,h')|\gtrsim\epsilon'$.
\end{enumerate} 
Notably, the above objectives need to and can be accomplished without incurring large regret.

If objective (a) is accomplished, the intermediate layer $h<H$ is not that different from the last layer $H$,
because actions induced by the distribution $G$ can be taken (i.e., $a_h=\pi_g(x_h)$ where $g\sim G$)
to estimate $|\ME(g,\pi_f,h)|$ for all $g\in\mathcal G$ without incurring large regret;
otherwise, we discard the current exploration policy $f$ and instead explore policy $g$, attempting to eliminate other policies
under $\pi_g$ in layer $h'$.
Because $h'$ is strictly larger than $h$, eventually we reach the last layer $H$ and the elimination procedure
reduces to the standard contextual bandit problem discussed in the previous paragraph.

\subsection{Additional assumptions and notations}

We make two additonal assumptions throughout this paper, both very mild.
\begin{assumption}[deterministic initial state]
The initial state $x_1$ is deterministic and known.
\label{asmp:deterministic-x1}
\end{assumption}

Assumption \ref{asmp:deterministic-x1} is in fact without loss of generality, because any MDP with $H$ layers can be transformed 
to another MDP with $H+1$ layers with only one state no reward in the first layer.

To state the second assumption we define ``restriction'' and ``concatenation'' of function approximators.
In particular, for any $f\in\mathcal F$ and $h\in[H]$, let $f_h:\mathcal X_h\times \mathcal A\to\mathbb R$
be the function restricted to states $x_h\in\mathcal X_h$ in layer $h$.
Define $\mathcal F_h := \{f_h: f\in\mathcal F\}$ for all such restrictions induced by $\mathcal F$.
\begin{assumption}[closedness under concatenation]
Let $f:\mathcal X\times\mathcal A\to\mathbb R$ be an arbitrary function approximation.
Suppose there exist $g_1\in\mathcal F_1,\cdots,g_H\in\mathcal F_H$ such that $f(x_h,a) =  g_h(x_h,a)$ for all $h\in[H]$ and 
$x_h\in\mathcal X_h$. Then $f\in\mathcal F$.
\label{asmp:closedness}
\end{assumption}

Generally speaking, Assumption \ref{asmp:closedness} requires the function class $\mathcal F$ to be large enough such that
it allows \emph{non-stationary} function approximation,
using essentially independent function approximators for each layer $h$.
Such an assumption is very mild because under episodic settings with finite $H$,
independent function approximation is almost always used among different layers since even the same state 
could lead to very different rewards in different layers.

\section{Related work}\label{sec:related-work}

Function approximation is an old idea in learning and planning of MDPs and other dynamic programming systems.
Under the reinforcement learning context, function approximation is typically used to learn the Q-functions of the optimal policy of an MDP, 
a method commonly referred to as \emph{Q-learning}.
\cite{watkins1989learning} initiated the study of Q-learning and proposed the first such algorithm known as \emph{fitted Q-iterations} (FQI),
which takes data collected from a fixed exploration policy and iteratively finds functions $f\in\mathcal F$ that minimize the least-square Bellman error $\mathbb E[|f(x,a)-r(x,a)-\mathbb E\max_{a'}f(x',a')|^2]$.
\cite{watkins1992q,tsitsiklis1994asynchronous} provide asymptotic convergence results for FQI with finite state-action spaces.
\cite{farahmand2016regularized,lazaric2012finite,antos2008learning,munos2008finite} gives finite-sample convergence guarantees of FQI, with \cite{chen2019information}
achieving the optimal sample complexity dependency on approximation error parameters.
Most analysis of FQI assumes the function class $\mathcal F$ is \emph{closed} under the Bellman update operator, and the exploration policy used to collect data satisfies certain \emph{low-concentratability} conditions,
neither of which is assumed in this paper.
In general, FQI could oscillate and diverge \citep{gordon1995stable}.

When no good exploration policy is known a priori, learning of MDPs becomes much more challenging as the exploration policy needs to constantly change to achieve good state coverage.
\cite{jin2018q} applies the idea of upper confidence bands (UCB) to obtain $\sqrt{n}$-regret for learning tabular MDPs,
which was later generalized to linear function approximation under strong linear transition probabilities assumptions \citep{jin2019provably}.
For problems with low-Bellman rank, the \textsf{OLIVE} algorithm provides sequentially exploration policies, with its computationally tractable variants in \citep{dann2018oracle,du2019provably} under more restrictive settings.
\cite{wen2017efficient} studies MDPs with \emph{deterministic} transition and reward functions.
\cite{du2019provably} studies MDPs with nearly deterministic transition functions under an additional ``gap'' assumption.

Theoretical analysis is also available for model-based and policy optimization type algorithms \citep{sidford2018near,fazel2018global,wainwright2019stochastic}.
Our algorithm also made use of methods for \emph{contextual bandit}, with several recent developments in \citep{dudik2011efficient,agarwal2014taming}.

For the special case of \emph{linear} function approximation, the works of \cite{sidford2018near,yang2019reinforcement,jin2019provably}
derive PAC-learning or polynomial regret upper bounds under the condition of completeness with respect to Bellman updates
as well as additional assumptions imposed on state transition probabilities.
The work of \cite{du2019provably} studies linear function approximation for mostly deterministic state transition problems.
It is also noted that the algorithms in \citep{sidford2018near,yang2019reinforcement,jin2019provably,du2019provably} are computationally efficient
and therefore easier to implement in practice.

\section{The Adaptive Value-function Elimination (\textsf{AVE}) algorithm}

We present the details of our algorithm in Algorithm~\ref{alg:main}, \ref{alg:eliminate},  \ref{alg:check}, and \ref{alg:identify}, where the exact values of parameters $n_i^{\text{eval}}$, $n_i$, $n_i^{\text{cb}}$ and $n_i^{\text{id}}$ can be found in Appendix \ref{app:concentration}.  Throughout the algorithm, we set 
\[
\epsilon_i:=1/2^i, \qquad \epsilon_i':=\epsilon_i/\lceil\log(|\MF|)+1\rceil, \qquad \phi_i:=\epsilon_i/(12\sqrt{M})
\]
for all $i \in \{0, 1, 2, \dots \}$. In the \Eliminate procedure, we define 
\[
W'_{P_k}(x,a):=(1-A\mu_k)\sum_{f\in \mathcal{G}}\mathbb{I}[\pi_f(x)=a]P_k(f)+\mu_k. 
\]

Below we describe each procedure at a high level.  

\begin{algorithm}[ht]
	\caption{$\Main(\delta, \epsilon, n)$}
	\label{alg:main}
	\begin{algorithmic}[1]
		\State $\mathcal{G} \leftarrow \mathcal{F}$ ;
		\While{\textsc{true}}
			\State $f\gets \arg\max_{f\in \mathcal{G}}f(x_1,\pi_f(x_1))$; \label{line:OFU}
			\For{$k=1,2,3,\dots, \lceil \log_2 {H/\epsilon}\rceil$} \label{line:main-3}
			    \State Run $\pi_{f}$, collect trajectories $\{(x_q^{p},a_q^p,r_q^p)_{q=1}^{H}\}_{p=1}^{n_k^{\text{eval}}};$\label{line:main_run}  \label{sample:1}
			    \State Estimate \vspace{-2ex} \begin{equation}\tilde{\ME}_k(f,\pi_f,h)=\frac{1}{n_k^{\text{eval}}}\sum_{p=1}^{n_k^{\text{eval}}}\left[f(x_{h}^p,a_{h}^p)-r_{h}^p-f(x_{{h}+1}^p,a_{{h}+1}^p)\right];\label{equ:main_concentration}\end{equation} \label{line:main-estimate}
			    \State $h\gets \arg\max_{h':1\le h'\le H}\left|\tilde{\ME}_k(f,\pi_f,h')\right|;$
			    \If{$\left|\tilde{\ME}_k(f,\pi_f,h)\right|>\epsilon_k$}\label{line:main_if}
			    	\State$\Eliminate(f,h,k)$;\label{line:main_call_eliminate} \textbf{break};
				\EndIf
			\EndFor 
			\State \textbf{if} {\Eliminate is not called during this iteration for all $k\le \lceil \log_2 H/\epsilon\rceil$} \textbf{then} \textbf{break};
		\EndWhile \label{line:main-12}
		\State Run policy $\pi_f$ for the remaining trajectories.\label{sample:0}
	\end{algorithmic}
\end{algorithm}

\paragraph{The \Main algorithm.} The \Main procedure takes three parameters: $\delta$ is the confidence level (so that the algorithm succeeds with probability at least $(1-\delta)$), $\epsilon$ is the target precision, and $n$ is the number of total trajectories to run. The algorithm will terminate early whenever the $n$ trajectories are used up.

As mentioned before, our algorithm is an elimination-based algorithm, where we repeatedly eliminate sub-optimal hypotheses in $\mathcal{F}$. Here we say a hypothesis is sub-optimal if it has a large Bellman error, or its induced policy has a sub-optimal value. In more details, in the \Main procedure, we repeatedly choose the hypothesis $f$ so that its induced policy has the best estimated value, which is in accordance with the optimism in face of uncertainty principle (OFU). From Line~\ref{line:main-3} to Line~\ref{line:main-12}, we estimate the Bellman error of $\pi_f$. If the total Bellman error is large, we find particular layer $h$ with large Bellman error and invoke \Eliminate at Line \ref{line:main_call_eliminate} to eliminate all functions in $\mathcal{G}$ with large Bellman error at layer $h$, given the roll-in policy $f$. Otherwise, we know that the real value of $f$ is close to its estimated value, which is almost optimal by the OFU principle, and therefore we can keep running $\pi_f$ for the remaining times.

We note that at Line~\ref{line:main-3}, a doubling trick is used to ensure that a policy $\pi_f$ with large regret (which is upper bounded by the Bellman error of $f$) is not executed for too many times. 

We also remark that the number of iterations made by the \textbf{while} loop is upper bounded by a volumetric argument adapted from \cite{Jiang2017ContextualDP}, the details of which are presented in Appendix~\ref{app:volume}.

\begin{algorithm}[ht]
	\caption{\Eliminate$(g,h,j)$}
	\label{alg:eliminate}
	\begin{algorithmic}[1]
		\For{$k=1,2,3,\dots,j$}
			\State Set 
			$\mu_{k}=\frac{\epsilon_{k}}{A}$
			and find distribution $P_k$ over $\mathcal{G}$, such that $\forall f\in \mathcal{G},$\label{line:find_distribution}
			\begin{equation}
				\mathop{\E}_{x\sim \MD_{g,h}}\left[\frac{1}{(1-A\mu_k)W_{P_k}(x,\pi_f(x))+\mu_k}\right]\le 110A.\label{equ:low_variance}
			\end{equation}
			\State $(c,g',h',k')\gets \Chk(\{(g\circ_h f,P_k(f))\}_{f\in \mathcal{G}},h,k-2)$; \label{line:eliminate-3}
			\State \textbf{if} ${c = \textsc{false}}$ \textbf{then} $\Eliminate(g',h',k'+1)$; and \textbf{return}; \label{line:eliminate-4}
			
			\State Repeatedly sample $f\sim P_k$, $\pi\gets \begin{cases}\pi_g\circ_h \pi_U,&\text{with probability }A\mu_k,\\
			\pi_g\circ_h \pi_{f}, &\text{with probability }1-A\mu_k,\end{cases}$ and run policy $\pi$ for $n_k^{\text{cb}}$ times; collect samples $\{x_h^p,a_h^p,r_h^p,x_{h+1}^p\}_{p=1}^{n_k^{\text{cb}}};$  \label{sample:3} \label{line:run_policy}
			\State \label{line:eliminate-line-5} Define \begin{equation}\tilde{\eta}_k(f,g,h)\defeq\frac{1}{n_k^{\text{cb}}}\sum_{p=1}^{n_k^{\text{cb}}}(r^p_h+f(x^p_{h+1},\pi_{f}(x^p_{h+1})))\frac{\mathbb{I}[\pi_{f}(x^p_h)=a^p_h]}{W'_{P_k}(x_h^P,a_h^P)}.\label{equ:eliminate_eta}\end{equation}
			\State \textsc{pesudo-learn} \label{line:bandit_elimination} \begin{equation}\mathcal{G}\gets \{f\in \mathcal{G}:\tilde{\eta}_k(f,g,h)\ge \tilde{\eta}_k({g},g,h)-(6H+1)\epsilon_k\};\end{equation} 
		\EndFor \label{line:eliminate-7}
		\State Repeatedly sample $f\sim P_j$, $\pi\gets \begin{cases}\pi_g\circ_h \pi_U,&\text{with probability }A\mu_j,\\
			\pi_g\circ_h \pi_{f}, &\text{with probability }1-A\mu_j,\end{cases}$ and run policy $\pi$ for $n_j$ times; collect samples $\{x_h^p,a_h^p,r_h^p,x_{h+1}^p\}_{p=1}^{n_j};$\label{sample:4}
		\State Estimate 
			\begin{equation}\tilde{\ME}_j(f,\pi_{g},h)\defeq\frac{1}{n_j}\sum_{p=1}^{n_j}\frac{\mathbb{I}[a_h^p=\pi_{f}(x_h^p)]}{W'_{P_j}(x_h^P,a_h^P)}\left(f(x_h^p,a_h^p)-r_h^p-f(x_{h+1}^p,\pi_f(x_{h+1}^p))\right)\label{equ:bellman_estimate}\end{equation}
		\State \textsc{learn} \label{line:bellman_elimination} 
			\begin{equation}\mathcal{G}\gets\left\{f\in \mathcal{G}:\left|\tilde{\ME}_j(f,\pi_{g},h)\right|\le \phi_j\right\};\end{equation}
	\end{algorithmic}
\end{algorithm}

\paragraph{The \Eliminate procedure.} We proceed to discuss the implementation of  $\Eliminate(g,h,j)$. As mentioned before, the procedure tries to estimate Bellman error $\ME(f,\pi_{g},h)$ for all $f\in \mathcal{G}$, and eliminate those with $|\ME(f,\pi_{g},h)|>\phi_j$ from the hypothesis space $\mathcal{G}$ (\textsc{learn} step at Line \ref{line:bellman_elimination}).

A key challenge here is that it is non-trivial to estimate the Bellman error for all $f\in \mathcal{G}$ with small regret. In \cite{Jiang2017ContextualDP}, the authors used a straightforward adaption of the importance sampling approach that uniformly samples an action. However, this only guarantees the sample complexity but could lead to a regret that is linear with $|\mathcal{G}|$, which is not affordable. To solve this problem,  we borrow the idea from contextual bandit literature \cite{dudik2011efficient}. More precisely, we look for a distribution $P_k$ over the hypothesis space $\mathcal{G}$ such that i) when applying importance sampling with $P_k$ as the sampled distribution, the estimation has low variance (see Eq.~\eqref{equ:low_variance}), where we elaborate how to implement Line~\ref{line:find_distribution} in Appendix~\ref{app:low_variance}; and ii) the regret when running the randomized policy according to $P_k$ is small.

We also need to perform exploration very carefully to avoid large regret. To achieve this, the \Eliminate procedure consists of two parts. The first part is before (and at) Line~\ref{line:eliminate-7}. In this part, we use a doubling trick (via loop variable $k$) and gradually eliminate value functions with \emph{low predicted performance} (which will be concretely defined in the later paragraph). This step makes sure that we do not run too many sub-optimal policies and is necessary for achieving lower regret. The second part consists of the lines after Line~\ref{line:eliminate-7}, where we estimate the Bellman error at layer $h$ for all hypotheses and perform the elimination for large Bellman error.

We now elaborate the first step. The \emph{predicted performance} of each function $f$ is measured by the function $\eta(f,g,h)$ defined as follows,
\begin{align}\label{eq:def-eta}
\eta(f,g,h)\defeq \E_{x\sim \MD_{g,h},x'\sim p(\cdot\mid x,\pi_{f}(x))}[r(x,\pi_f(x))+f(x',\pi_f(x'))].
\end{align}
To help illustration, let us consider the special case when $h=H$ (i.e., $h$ is the last layer). In this case, $\Chk$ always returns $c = \textsc{true}$ and we can temporarily ignore Line~\ref{line:eliminate-3} and Line~\ref{line:eliminate-4}. Since $h = H$, we have $f(x',\pi_f(x'))=0$ for all $f\in \mathcal{G}$. Therefore $\eta(f,g,h)$ is the expected reward of the policy induced by $g\circ_h f$ at layer $H$. The \textsc{pseudo-learn} step eliminates the policies that perform significantly worse than $g$, which means that the regret of running any remaining policies is comparable to that of $g$. Therefore, any distribution over the remaining policy would achieve affordable regret.  We also note that the optimal policy $\pi_{f^*}$ achieves optimal value no matter what roll-in policy $g$ is. Thus, $f^*$ will not be eliminated during this process. 

Now we move to the more general case when $h<H$. To avoid large regret, we are really interested in the \emph{actual performance} of the random policy $\pi \sim \{(g\circ_h f,P_k(f))\}_{f\in \mathcal{G}}$, which is the expected reward of the policy induced by $g \circ_{h} f$ at and after layer $h$, i.e., 
\[
\E_{f\sim P_k}   \E_{x\sim \MD_{g,h},x'\sim p(\cdot\mid x,\pi_{f}(x))} [r(x,\pi_{f}(x))+V^{\pi_f}(x')].
\]
By comparing this expression with the definition of $\eta$ in Eq.~\eqref{eq:def-eta}, we see that since we have no direct access to $V^{\pi_f}(x')$, we have to use $f(x',\pi_f(x'))$ (i.e., the value predicted by the hypothesis) instead. Thus, we need to make sure $\E_{f\sim P_k}f(x',\pi_f(x'))$ well approximates $\E_{f\sim P_k}V^{\pi_f}(x')$, which is done by a recursive call to \Chk at Line~\ref{line:eliminate-3}. We only proceed when \Chk returns \textsc{true} (i.e., certifying the two values are close). Otherwise,  \Chk returns $c = \textsc{false}$ and also identifies a layer $h' > h$ and a function $g'$ from the support of $\{(g\circ_h f,P_k(f))\}_{f\in \mathcal{G}}$, such that $g \circ_{h} g'$ has large Bellman error at layer $h'$. Now we recursively call $\Eliminate$ with roll-in function $g'$ and layer $h'$ instead. We note that since the $h$ parameter keeps increasing along the recursive path, the depth of the recursion can be properly upper bounded.

We finally explain the second part of \Eliminate. We again use importance sampling with distribution $P_j$. Thanks to the first part, we know that the induced policies by the hypotheses sampled from the distribution $P_j$ do not incur too large regret. Therefore, we are able to sample to pull sufficient samples from $P_j$, estimate the Bellman error at layer $h$ for all $f \in \mathcal{G}$ based on the roll-in policy $g$, and perform the elimination for large Bellman error.

\begin{algorithm}[ht]
	\caption{\Chk$(G,h,j)$}
	\label{alg:check}
	\begin{algorithmic}[1]
		\For{$k=1,2,\cdots,j$}			
			\State For $p = 1, 2, 3, \dots, n_k^{\text{eval}}$, sample $f^p \sim G$ and run $\pi_{f^p}$, collect trajectories $\{(x_q^{p},a_q^p,r_q^p)_{q=1}^{H},f^p\}_{p=1}^{n_k^{\text{eval}}};$\label{line:check_mix_run}  \label{sample:5}
			\State Estimate \begin{equation}\hat{\ME}_k(G,h')=\frac{1}{n_k^{\text{eval}}}\sum_{p=1}^{n_k^{\text{eval}}}\left[f^p(x_{h'}^p,a_{h'}^p)-r_{h'}^p-f^p(x_{{h'}+1}^p,a_{{h'}+1}^p)\right];\label{equ:check_mix_concentration}\end{equation} \label{line:check-estimate}
			\If{$\left|\sum_{h'=h+1}^{H}\hat{\ME}_k(G,h')\right|>(H-h)\epsilon_k$}\label{line:check_if}
				\State $(g_r,h_r,k_r)\gets\Identify(G,h,k)$;\label{line:call_identify}
				\State \Return$(\textsc{false},g_r,h_r,k_r)$;
			\EndIf
		\EndFor
		\State\Return $(\textsc{true},*)$;
	\end{algorithmic}
\end{algorithm}

\paragraph{The \Chk procedure.} Let us first define the expected Bellman error of a distribution $G$ of hypothesis functions at layer $h$, as follows,
\begin{equation}\label{equ:average_bellman_error}
    \ME(G,h)\defeq \E\nolimits_{f\sim G}\ME(f,\pi_f,h).
\end{equation}
When $\Chk(G, h, j)$ is invoked, the procedure either returns $(\textsc{true}, *)$ and certifies that $\sum_{h' = h+1}^{H} \ME(G, h')$ is close to zero (the closeness is defined by the parameter $j$), or identifies a layer $h' \in \{h, h+1, h+2, \dots, H\}$ and a function $g$ from the support of $G$ such that the Bellman error $\ME(g, \pi_g, h')$ is large (and the magnitude is quantified by $k'$). 

To achieve this goal, we first use a doubling trick (the $k$ loop) to control the regret. In each iteration of the loop, we estimate $ \ME(G,h')$ for each $h' \in \{h+1, h+2, \dots, H\}$, up to precision $\epsilon_k$. Once we figure out that  $\sum_{h' = h+1}^{H} \ME(G, h')$ is significant, we call \Identify to find out a specific function $g$ and a layer $h' \in \{h+1, h+2, \dots, H\}$ from the support of $G$ such that $\ME(g, \pi_g, h')$ is significant.



\begin{algorithm}[ht]
	\caption{\Identify$(G,h,k)$}
	\label{alg:identify}
	\begin{algorithmic}[1]
		\State $\textsc{step}\gets 0$;
		\While{$|\supp(G)|>1$}\label{line:identify_while}
			\State $\textsc{step}\gets \textsc{step}+1;$
			\State Choose any subset $\mathcal{G}'\subset \supp(G)$ such that $\lfloor|\supp(G)|/2\rfloor\le |\mathcal{G}'|\le \lceil|\supp(G)|/2\rceil$;\label{line:find_subset}
			\State $G_1\gets G\rvert_{\mathcal{G}'}, G_2\gets G\rvert_{\supp(G)\setminus\mathcal{G}'};$
			\State $c\gets \textsc{false};$
			\For{$l=1,2,3,\dots,k$}
			\State For $p = 1, 2, 3, \dots, n_l^{\text{id}}$, sample $f^p \sim G_1$ and run $\pi_{f^p}$, collect trajectories $\{(x_q^{p},a_q^p,r_q^p)_{q=1}^{H},f^p\}_{p=1}^{n_l^{\text{id}}};$\label{sample:6}
			\State Estimate \vspace{-2.5ex} \begin{equation}\hat{\ME}_l(G_1,h')=\frac{1}{n_{l}^{\text{id}}}\sum_{p=1}^{n_{l}^{\text{id}}}\left[f^p(x_{h'}^p,a_{h'}^p)-r_{h'}^p-f^p(x_{h'+1}^p,a_{h'+1}^p)\right];\label{equ:identify_concentration}\end{equation}
			\If {$\left|\sum_{h'=h+1}^{H}\hat{\ME}_l(G_1,h')\right|>(H-h)\left(\epsilon_{l+1}-(\textsc{step}-0.5)\cdot \epsilon'_{l+2}\right)$}\label{line:identify_if_1}
			    \State $G\gets G_1; k\gets l; c\gets \textsc{true};$ \label{line:case_1} \textbf{break};
			\EndIf
		    \EndFor
		    \State \textbf{if} ${c = \textsc{false}}$ \textbf{then}  $G\gets G_2$;
		\EndWhile \label{line:identify_while_end}
		
		\State Let $g_r$ be the only element in $\supp(G)$;
		\For{$l=1,2,3,\dots,k$} \label{line:identify_for}
			\State For $p = 1, 2, 3, \dots, n_l^{\text{id}}$, run $\pi_{g_r}$, collect trajectories $\{(x_q^{p},a_q^p,r_q^p)_{q=1}^{H}\}_{p=1}^{n_l^{\text{id}}};$\label{sample:7}
			\State Estimate \vspace{-2.5ex}\begin{equation}\tilde{\ME}_l(g_r,\pi_{g_r},h')=\frac{1}{n_{l}^{\text{id}}}\sum_{p=1}^{n_{l}^{\text{id}}}\left[g_r(x_{h'}^p,a_{h'}^p)-r_{h'}^p-g_r(x_{h'+1}^p,a_{h'+1}^p)\right];\label{equ:ident_final}\end{equation}
			\State $h_r\gets \arg\max_{h':h+1\le h'\le H}\left|\tilde{\ME}_l(g_r,\pi_{g_r},h')\right|;$
			\If {$\left|\tilde{\ME}_l(g_r,\pi_{g_r},h_r)\right|>\epsilon_{l+2}+0.5\epsilon_{l+2}'$}\label{line:identify_return_if}
				\State $k_r\gets l$; \textbf{return} $(g_r, h_r, k_r);$\label{line:identify_return}
			\EndIf
		\EndFor		
		
	\end{algorithmic}
\end{algorithm}

\paragraph{The \Identify procedure.} $\Identify(G,h,k)$ is called when we know that  $|\sum_{h' = h+1}^{H}  \ME(G, h')| \gtrsim (H - h)\epsilon_k$, and the procedure will find a value function $g_r \in \supp(G)$, a layer $h_r \in \{h +1, h+2, \dots, H\}$, and a precision parameter $k_r$ such that the Bellman error of $g_r$ at layer $h_r$ is large (more specifically, $\ME(g_r, \pi_{g_r}, h_r \gtrsim \epsilon_{k_r}$).

Since we cannot afford examining the Bellman error of every function in $\supp(G)$, \Identify finds the desired value function via binary search, which is done by the \textbf{while} loop from Line~\ref{line:identify_while} to Line~\ref{line:identify_while_end}. More precisely, every time we split $\supp(G)$ into two parts: $\mathcal{G}'$ and $\supp(G)\setminus \mathcal{G}'$ and define the induced distributions $G_1$ and $G_2$ respectively. If the expected Bellman error of $G$ is large, then we know that at least one of $G_1$ and $G_2$ has large expected Bellman error. We learn the expected Bellman error of $G_1$ (where the doubling trick is used to control the regret). We iterate the process with $G_1$ if its expected Bellman error is large, and with $G_2$ otherwise, until only one function is left in the support of the distribution, which is identified as $g_r$.

Finally, from Line~\ref{line:identify_for} to the end of the procedure, we use a doubling trick to learn $\ME(g_r, \pi_{g_r}, h')$ for each $h' \in \{h+1, h+2, \dots, H\}$ while controlling the regret, and find out a layer $h_r$ so that $\ME(g_r, \pi_{g_r}, h') \gtrsim \epsilon_{k_r}$.

\section{The analysis} \label{sec:analysis}

In this section, we provide theoretical analysis for our algorithm.  We first prove the functionalities for each procedure, and then combine these lemmas to upper bound the number of trajectories collected by the algorithm and the expected total regret of the trajectories. Let $\mathcal{Z}$ be the event that all empirical estimations in Eqs.~\eqref{equ:main_concentration}, \eqref{equ:eliminate_eta}, \eqref{equ:bellman_estimate}, \eqref{equ:check_mix_concentration}, \eqref{equ:identify_concentration}, and \eqref{equ:ident_final} concentrate to their real values, where in Appendix~\ref{app:concentration} we give the formal definition of $\mathcal{Z}$ and show that $\Pr[\mathcal{Z}]\ge 1-\delta$. The whole analysis only focuses on the desired situation when $\mathcal{Z}$ happens. For a policy $\pi$, we define $V^{\pi}$ as a shorthand for $V^{\pi}(x_1).$ We also define 
\begin{align}
P&:=MH\zeta/\epsilon\\
L & := \lceil\log_2 (H/\epsilon)\rceil, \label{eq:def-L}\\
\iota & :=\log(\zeta/(2\phi_L))/\log(5/3), \label{eq:def-iota} \\
C & := LHM \iota. \label{eq:def-C} 
\end{align}
for simplicity.

\subsection{Analysis for sub-procedures}\label{sec:per-procedure}

Given a distribution $G$, we define the property $\mathcal{P}(G, h, \epsilon)$ for all $h \in \{1, 2, 3, \dots, H\}$ and $\epsilon > 0$ as the conjunction of the following conditions.
\begin{enumerate}[(a)]
	\item There exists a value function $g$ and a distribution $P$ over $\mathcal{F}$, such that $G$ can be expressed in the following form: \label{item:calP-a}
\[
G=\{(g\circ_h f,P(f))\}_{f\in \MF} ;
\] 
	\item $V^{\pi_g}\ge V^*-(24H+4)h\epsilon$; \label{item:calP-b}
	\item $\left|\sum_{h'=h+1}^{H}\ME(g,\pi_g,h')\right|\le 6H\epsilon$; \label{item:calP-c}
	\item $\eta(f,g,h)\ge \eta(g,g,h)-(12H+4)\epsilon, \forall f\in \supp(P)$. \label{item:calP-d}
\end{enumerate}

The following lemma characterizes the guarantees of the \Identify procedure.

\begin{lemma}[Main Lemma for \Identify]\label{lem:main_identify}
Let $\mathcal{Q}_{\textsc{ID}}(G, h, k)$ be the condition such that $\left|\sum_{h'=h+1}^{H}\ME(G,h')\right|\ge (H-h)\epsilon_{k+1}$. Suppose $\mathcal{Q}_{\textsc{ID}}(G, h, k)$ holds for an invocation of $\Identify(G,h,k)$, and let $(g_r,h_r,k_r)$ be the returned tuple of this invocation. Under event $\mathcal{Z}$, for the first $L^2C$ times \Identify is called, we have that
\begin{enumerate}[(a)]
	\item \Identify returns at Line~\ref{line:identify_return}, with $g_r\in \supp(G)$, $h_r\in \{h+1,h+2,\cdots,H\}$, and $0\le k_r\le k$;\label{item:identify-a}
	\item $\left|\ME(g_r, \pi_{g_r}, h_r)\right|\ge \epsilon_{k_r+2}$;\label{item:identify-b}
	\item $\left|\ME(g_r,\pi_{g_r},h')\right|\le 3\epsilon_{k_r},\quad \forall h'\in \{h+1,h+2,\cdots,H\}$;\label{item:identify-c}
	\item If we additionally have $\mathcal{P}(G,h,\epsilon_k)$, then the regret incurred by this invocation is bounded by \[
	O(H^2\ln^3(|\MF|)\ln(P/\delta)/\epsilon_k) . \]\label{item:identify-d}
\end{enumerate}
\end{lemma}

The reason why the statement only focus on the first $L^2C$ invocations of \Identify is to make sure concentration events hold with high probability. As we prove later (see Lemma~\ref{lem:number_of_invoke}), the number of invocations is bounded by $L^2C$ with high probability. We also present the main lemma for \Chk as follows.
\begin{lemma}[Main Lemma for \Chk]\label{lem:main_check}
For any invocation of $\Chk(G, h, j)$, Under event $\mathcal{Z}$, for the first $L^2C$ times \Chk is called, we have
\begin{enumerate}[(a)]
	\item If the procedure returns $(\textsc{true}, *)$, then \[\left|\sum_{h'=h+1}^{H}\ME(G,h')\right|\le 1.5(H-h)\epsilon_j. \]  \label{item:lem-main-check-a}
	\item If the procedure returns $(\textsc{false}, g_r,h_r,k_r)$, then 
	\[
	\left|\ME(g_r, \pi_{g_r}, h_r)\right|\ge \epsilon_{k_r+2}, \qquad \left|\ME(g_r,\pi_{g_r},h')\right|\le 3\epsilon_{k_r},\quad\forall h'\in \{h+1,h+2,\cdots,H\} \]
	and $g_r\in \supp(G), h_r\in \{h+1,h+2,\cdots,H\}, 0\le k_r\le j $. \label{item:lem-main-check-b}
	\item If we addtionally have $\mathcal{P}(G,h,\epsilon_j)$, then the regret of this invocation is bounded by $O(H^2\ln^3(|\MF|)\ln(P/\delta)/\epsilon_j)$. \label{item:lem-main-check-c}
\end{enumerate}
\end{lemma}

To introduce the guarantees for \Eliminate, we first define the condition $\mathcal{Q}_{\textsc{ELIM}}(g, h, j)$ as the conjunction of the following items.
\begin{enumerate}[(a)]
	\item $V^{\pi_g}\ge V^*-(24H+4)h\epsilon_j$; \label{item:Q-elim-a}
	\item $\left|\sum_{h'=h+1}^{H}\ME(g,\pi_g,h')\right|\le 6(H-h)\epsilon_j$;  \label{item:Q-elim-b}
	\item $\left|\ME(g,\pi_g,h)\right|\ge \epsilon_j/2$. \label{item:Q-elim-c}
\end{enumerate}

Now we state the main lemma for \Eliminate.
\begin{lemma}[Main Lemma for \Eliminate]\label{lem:eliminate_main}
For an invocation of $\Eliminate(g, h, j)$, suppose $\mathcal{Q}_{\textsc{ELIM}}(g, h, j)$ holds. Then, under event $\mathcal{Z}$, for the first $LC$ times \Eliminate is called, we have that, 
\begin{enumerate}[(a)]
	\item\label{item:eliminate_main-a}
  \textsc{learn} step will be executed exactly once during this invocation; and when \textsc{learn} step is executed, Eq.~\eqref{equ:iteration_condition_1} and Eq.~\eqref{equ:iteration_condition_2} holds;
	\item the regret of this invocation is bounded by $c_{\rm ELIM} MAH^2\ln^3(|\MF|)\ln(P/\delta)/\epsilon_j $, where $c_{\rm ELIM}$ is a large enough universal constant;  \label{item:eliminate_main-b}
	\item if $f^*\in \mathcal{G}$ before this invocation, then $f^*$ will not be eliminated. \label{item:eliminate_main-c}
\end{enumerate}
\end{lemma}

Proofs of Lemma~\ref{lem:main_identify} and Lemma~\ref{lem:main_check} are deferred to Appendix~\ref{app:analysis}. We show the proof of Lemma~\ref{lem:eliminate_main} in Section~\ref{sec:pf_eliminate_main}.

\subsection{Regret analysis}
In this section, we prove the $\tilde{O}(\sqrt{n})$ regret bound for our \Main algorithm. Lemma~\ref{lem:eliminate_main} already shows the regret bound per invocation of $\Eliminate(g, h, j)$ is small, under the condition $\mathcal{Q}_{\textsc{ELIM}}(g, h, j)$. In order to upper bound the overall regret, we need first to show that the condition $\mathcal{Q}_{{\rm ELIM}}(g, h, j)$ is met every time the \Main algorithm calls \Eliminate. Then we upper bound the number of invocations made to \Eliminate. We finally analyze the doubling/halving trick in the main algorithm, and stitch all parts together to achieve the desired $\tilde{O}(\sqrt{n})$ regret bound.

\paragraph{Verifying the condition $\mathcal{Q}_{{\rm ELIM}}(g, h, j)$.} Here we verify that the condition $\mathcal{Q}_{\textsc{ELIM}}(g, h, j)$ is met whenever \Main calls \Eliminate. The policy decomposition lemma in \cite{Jiang2017ContextualDP} plays a critical role in our analysis, as it connects Bellman error to the gap between the actual value $V^{\pi_f}(x_1)$ and the predicted value $f(x_1,\pi_f(x_1))$ for any hypothesis $f$. For completeness, we state the lemma as follows.
\begin{lemma}[Policy loss decomposition \cite{Jiang2017ContextualDP}] \label{lem:bellman_error} For any value hypothesis $f:\MX\times \MA\to [0,1],$
\begin{equation}f(x_1,\pi_f(x_1))-V^{\pi_f}(x_1)=\sum_{h=1}^{H}\ME(f,\pi_f,h).\label{equ:bellman_error}\end{equation}
\end{lemma}

Now we show that the desired $\mathcal{Q}_{\textsc{ELIM}}(g, h, j)$ is condition is met when the desired event $\mathcal{Z}$ happens.
\begin{lemma}\label{lem:main_calls_eliminate}
Under $\mathcal{Z}$, $\mathcal{Q}_{\textsc{ELIM}}(g, h, j)$ holds for the first $C$ times that \Main calls $\Eliminate(g, h, j)$.
\end{lemma}
\begin{proof}
We will prove the following statements for $i \in \{1, 2, 3, \dots, C\}$ by induction,
\begin{enumerate}[(i)]
\item before the $i$-th  invocation of \Eliminate, we have $f^* \in \mathcal{G}$; \label{item:proof-main-calls-eliminate-a}
\item $\mathcal{Q}_{\textsc{ELIM}}(g, h, j)$ holds for the $i$-th invocation of $\Eliminate(g, h, j)$ by \Main. \label{item:proof-main-calls-eliminate-b}
\end{enumerate}
Note that statement (\ref{item:proof-main-calls-eliminate-a}) for $i = 1$ is guaranteed by Assumption~\ref{assumption:realizability}; for $i > 1$, it is derived by statement (\ref{item:proof-main-calls-eliminate-b}) with $i - 1$ and Lemma~\ref{lem:eliminate_main} (\ref{item:eliminate_main-c}). Therefore, we only need to show that statement (\ref{item:proof-main-calls-eliminate-a}) implies statement (\ref{item:proof-main-calls-eliminate-b}) for every $i \in \{1, 2, 3, \dots, C\}$ to complete the proof.

Now suppose $f^*\in \mathcal{G}$ before \Main calls $\Eliminate(f,h,k).$ When \Main calls $\Eliminate(f,h,k)$, the if-condition at Line~\ref{line:main_if} does not hold for function $f$ at iteration $k-1$ of the \textbf{for} loop. Therefore, $$\left|\tilde{\ME}_{k-1}(f,\pi_f,h')\right|\le \epsilon_{k-1},\quad\forall h'\in[H].$$ By the event $\mathcal{Z}$ (specifically, Eq.~\eqref{equ:conc_main}), we have that for any $h'\in [H],$
\begin{equation}\label{verify:cond_2}
\left|\ME(f,\pi_f,h')\right|\le 1.5\epsilon_{k-1}=3\epsilon_k.
\end{equation} 
It follows from Eq.~\eqref{equ:bellman_error} that $V^{\pi_f}\ge f(x_1,\pi_f(x_1))-3H\epsilon_k.$ Since $f$ is chosen by OFU principle and by induction hypothesis $f^*\in \mathcal{G}$, we get $f(x_1,\pi_f(x_1))\ge f^*(x_1,\pi_{f^*}(x_1))=V^*.$ Thus, 
\begin{equation}\label{verify:cond_1}
V^{\pi_f}\ge V^*-3H\epsilon_k.
\end{equation}
When \Main calls $\Eliminate(f,h,k)$, the if-condition at Line~\ref{line:main_if} is true for function $f$ at iteration $k$, which means that $\left|\tilde{\ME}_{k}(f,\pi_f,h)\right|\ge \epsilon_k.$ By the event $\mathcal{Z}$ (specifically, Eq.~\eqref{equ:conc_main}), we get 
\begin{equation}\label{verify:cond_3}
\left|\ME(f,\pi_f,h)\right|\ge \epsilon_k/2.
\end{equation}
Now the condition $\mathcal{Q}_{\textsc{ELIM}}(f, h, k)$ follows from Eq.~\eqref{verify:cond_1}, Eq.~\eqref{verify:cond_2}, and Eq.~\eqref{verify:cond_3}. 
\end{proof}

\paragraph{Bounding the number of invocations.}
The following lemma upper bounds the total number of \textsc{learn} steps executed in \Eliminate. The lemma is adapted from the volumetric argument in \cite{Jiang2017ContextualDP}, and is proved in Appendix~\ref{app:volume}. 
\begin{lemma}\label{lem:iteration_complexity}  For any $j$, if $\tilde{\ME}_j(f,\pi_{g},h)$ defined by Eq. \eqref{equ:bellman_estimate} satisfies 
\begin{equation}\label{equ:iteration_condition_1}
\left|\tilde{\ME}_j(f,\pi_{g},h)-\ME(f,\pi_{g},h)\right|\le \phi_j,\quad \forall f\in \mathcal{G},
\end{equation}
and whenever Line \ref{line:bellman_elimination} of Algorithm \ref{alg:eliminate} is executed,
\begin{equation}\label{equ:iteration_condition_2}
\left|\ME(g,\pi_{g},h)\right|\ge \epsilon_j/2=6\sqrt{M}\phi_j.
\end{equation}
Then for any $j$ and $h$, \textsc{learn} step will be executed at most $M\log(\zeta/(2\phi_j))/\log(5/3)$ times. And the optimal value function $f^*$ will never be eliminated.
\end{lemma}

Now we can bound the number of invocations to \Eliminate.

\begin{lemma}\label{lem:number_of_invoke} Under $\mathcal{Z}$, the procedure \Eliminate is called no more than $C$ times by \Main.
\end{lemma}
\begin{proof} 
Note that under the event $\mathcal{Z}$, by Lemma~\ref{lem:main_calls_eliminate}, the condition $\mathcal{Q}_{\textsc{ELIM}}(g, h, j)$ holds before every time of the first $C$ times that $\Eliminate(g, h, j)$ is called by \Main. 
As a result of statement (\ref{item:eliminate_main-a}) of Lemma~\ref{lem:eliminate_main}, conditions in Lemma~\ref{lem:iteration_complexity} holds. Therefore, \textsc{learn} step will be executed by no more than $M\iota$ times for each $h\in [H]$ and $j\in [L]$, which, in total, is at most $HLM\iota = C$ times. Since for every time \Main calls \Eliminate, the \textsc{learn} step is executed exactly once (by Lemma~\ref{lem:eliminate_main} (\ref{item:eliminate_main-a})), we upper bound the number of times that \Main calls \Eliminate by $C$.
\end{proof}

\paragraph{Regret for \Main.}

The following lemma controls the regret for each iteration $k$ of the \textbf{for} loops in \Main.
\begin{lemma}\label{lem:main_iteration}
Under the event $\mathcal{Z}$, at the $k$-th iteration of the \textbf{for} loop, the policy $\pi_f$ run in Line~\ref{line:main_run} of \Main is $3H\epsilon_k$-optimal. That is, 
$$V^{\pi_f}(x_1)\ge V^*(x_1)-3H\epsilon_k.$$
\end{lemma}
\begin{proof}
Since the algorithm proceeds to the $k$-th iteration, we have $|\tilde{\ME}_{k-1}(f,\pi_f,h)|\le \epsilon_{k-1}$ for every $h\in [H]$. By the event $\mathcal{Z}$ (specifically, Eq.~\eqref{equ:conc_main}), we get that $|\ME(f,\pi_f,h)|\le 1.5\epsilon_{k-1}=3\epsilon_k$ for every $h\in [H]$. It follows from Lemma~\ref{lem:bellman_error} that $V^{\pi_f}(x_1)\ge f(x_1,\pi_f(x_1))-3H\epsilon_k.$ Since $f$ is chosen according to the OFU principle (Line~\ref{line:OFU}), we have $f(x_1,\pi_f(x_1))\ge f^*(x_1,\pi_{f^*}(x_1))=V^*(x_1)$. Stitching the inequalities together we get, $$V^{\pi_f}(x_1)\ge f(x_1,\pi_f(x_1))-3H\epsilon_k\ge f^*(x_1,\pi_{f^*}(x_1))-3H\epsilon_k=V^*(x_1)-3H\epsilon_k.$$
\end{proof}
Now, we upper bound the regret incurred by running $\pi_f$ at Line~\ref{line:main_run}  during a single iteration of the outer \textbf{while} loop by,
\[
\sum_{k=1}^{L}3H\epsilon_kn_k^{\text{eval}}\lesssim \sum_{k=1}^{L}\frac{H\ln(P/\delta)}{\epsilon_k}\lesssim \frac{H\ln(P/\delta)}{\epsilon_L}\lesssim \frac{H^2\ln(P/\delta)}{\epsilon}.
\]
By Lemma~\ref{lem:number_of_invoke}, the \textbf{while} loop in \Main will be executed at most $C$ times. Therefore the overall regret incurred by running $\pi_f$ at Line~\ref{line:main_run}  is bounded by,
\begin{align}
C \cdot \frac{H^2\ln(P/\delta)}{\epsilon} \lesssim \frac{H^3M\ln^2(P)\ln(P/\delta)}{\epsilon}. \label{eq:final-regret-sample}
\end{align}

We then focus on the regret incurred by the invocations to \Eliminate by \Main. Under the event $\mathcal{Z}$, Lemma~\ref{lem:eliminate_main} shows that the regret incurred by each invocation of \Eliminate is bounded by
\[
O(MAH^2\ln^3(|\MF|)\ln(P/\delta)/\epsilon_k)\lesssim MAH^2\ln^3(|\MF|)\ln(P/\delta)/\epsilon_L \leq MAH^3\ln^3(|\MF|)\ln(P/\delta)/\epsilon.
\]
By Lemma~\ref{lem:number_of_invoke}, the number of invocations is bounded by $C$. Thus, the overall regret incurred by the calls to $\Eliminate$ is bounded by,
\begin{align}
C\cdot\frac{MAH^3\ln^3(|\MF|)\ln (P/\delta)}{\epsilon}=O\left(\frac{M^2AH^4\ln^2(P) \ln^3(|\MF|)\ln(P/\delta)}{\epsilon} \right).  \label{eq:final-regret-elim}
\end{align}

Finally, when the \textbf{while} loop terminates, we have $\left|\ME_{L}(f,\pi_f,h)\right|\le \epsilon_L$ for all $h\in [H]$, and 
\[
f=\mathop{\arg\max}_{f\in \mathcal{G}}f(x_1,\pi_f(x_1)) .
\]
Therefore, under the event $\mathcal{Z}$, we have that $V^{\pi_f}(x_1)\ge V^*(x_1)-1.5H\epsilon_L$, and the regret  incurred by running $\pi_f$ (for at most $n$ times) at Line~\ref{sample:0} is upper bounded by
\begin{align}
1.5H\epsilon_L \cdot n \lesssim n \epsilon .   \label{eq:final-regret-finalrun}
\end{align}

Combining the regret upper bounds in Eq.~\eqref{eq:final-regret-sample}, Eq.~\eqref{eq:final-regret-elim}, and  Eq.~\eqref{eq:final-regret-finalrun}, and the probability upper bound Eq.~\eqref{eq:calM-high-probability}, we have our main theorem.

\begin{theorem}\label{thm:main}
For any $\epsilon>0$ and $\delta>0$, with probability at least $1-\delta$, the overall regret running \Main for $n$ trajectories with parameter $\epsilon$ is bounded by
\[
{O}\left(M^2AH^4 \ln^2(P)\ln^3(|\MF|)\ln(P/\delta)/\epsilon+n\epsilon\right) .
\]
\end{theorem}
If we choose $\delta = 1/(nH)$ and $\epsilon=\sqrt{\frac{M^2AH^4\ln^3 (P) \ln^3(|\MF|)}{n}}$ for any given $n$, we have the following corollary.
\begin{corollary}
The expected regret of our algorithm for $n$ trajectories is upper bounded by
\[\tilde{O}\left(\sqrt{M^2AH^4n\ln^3 |\MF|}\right), \]
where the $\tilde{O}(\cdot)$ hides poly-logarithmic factors in $M, A, H, \zeta$, and $n$. 
\end{corollary}

\subsection{Proof of Lemma~\ref{lem:eliminate_main}}\label{sec:pf_eliminate_main}

Note that since $k' \leq k-2 \leq j - 2$ after Line~\ref{line:eliminate-3} (by Lemma~\ref{lem:main_check} (\ref{item:lem-main-check-b})), we have $k' + 1 \leq j -1$ at Line~\ref{line:eliminate-4}. Therefore, the $j$ parameter monotonically decreases as the \Eliminate recursively calls itself. Therefore, we have the following simple lemma.
\begin{lemma}
For any $h\in [H]$ and $j$, $\Eliminate(g, h, j)$ recursively calls itself by at most $j$ times, and calls \Chk by at most $j^2$ times during the whole recursion.
\end{lemma}

Also due to the monotonicity, we prove Lemma~\ref{lem:eliminate_main} by applying induction on the parameter $j$. The base case is that $j \leq 0$, where one can easily verify the correctness of the lemma.

Now suppose that  Lemma~\ref{lem:eliminate_main} is true for all $j'< j$, and consider an invocation $\Eliminate(g,h,j)$. We prove the three statements in the lemma as follows.

\paragraph{Proof of statement (\ref{item:eliminate_main-b}).} First we bound the suboptimality gap of policy $\pi$ run at Line~\ref{sample:3} of \Eliminate (where the formal statement to establish is Eq.~\eqref{eq:proof-eliminate-main-good-pi}), so that we can upper bound the regret incurred at Line~\ref{sample:3} and Line~\ref{sample:4}. Then we upper bound the regret incurred by the call to \Chk at Line~\ref{line:eliminate-3}. We also upper bound the regret incurred by the recursive call to \Eliminate itself at  Line~\ref{line:eliminate-4} via induction, to complete the proof.

Informally, the suboptimality gap of policy $\pi$ comes from the combination the following properties,
\begin{itemize}
    \item the roll-in policy $g$ is $\Omega(\epsilon_j)$-optimal;
    \item the action given by $\pi$ at the $h$-th layer has good predicted value (i.e., $\E_{f\sim P_k}\left[\eta(f,g,h)\right]\ge \eta(g,g,h)-O(\epsilon_k)$ for the $k$-th iteration); 
    \item the predicted value $\E_{f\sim P_k}\left[\eta(f,g,h)\right]$ at level $h$ is close to the true value $\E_{f\sim P_k}\left[V^{\pi_f}_{h}\right]$.
\end{itemize}

The following lemma is a generalization of Lemma~\ref{lem:bellman_error}.
\begin{lemma} For any distribution $G\in \Delta(\MF)$ and any layer $h'\in [H],$\label{lem:expected_bellman_error}
\begin{equation}\label{equ:expected_bellman_error}
\E_{f\sim G}\E_{x\sim \MD_{f,h+1}}\left[f(x,\pi_f(x))\right]-\E_{f\sim G}\E_{x\sim \MD_{f,h+1}}\left[V^{\pi_f}_{h+1}(x)\right]=\sum_{h'=h+1}^{H}\ME(G,h).
\end{equation}
\end{lemma}
\begin{proof}
For any $f\in \MF$, we have
\begin{align*}
&\E_{x'\sim \MD_{f,h+1}}\left[f(x',\pi_f(x'))-V^{\pi_f}_{h+1}(x')\right]\\
=&\E_{x'\sim \MD_{f,h+1}}\E_{x''\sim p(\cdot\mid x',\pi(x'))}\left[f(x',\pi_f(x'))-r(x,\pi_f(x))-f(x'',\pi_f(x''))+\left(f(x'',\pi_f(x''))-V^{\pi_f}_{h+2}(x'')\right)\right]\\
=&\ME(f,\pi_f,h+1)+\E_{x''\sim \MD_{f,h+2}}\left[f(x'',\pi_f(x''))-V^{\pi_f}_{h+2}(x'')\right].
\end{align*}
Keep unrolling the last term for $h+2,h+3, \dots,H$, and we have 
$$\E_{x'\sim \MD_{f,h+1}}\left[f(x',\pi_f(x'))\right]-\E_{x'\sim \MD_{f,h+1}}\left[V^{\pi_f}_{h+1}(x')\right]=\sum_{h'=h+1}^{H}\ME(f,\pi_f,h').$$
Take expectation for $f\sim G$, and we prove Eq.~\eqref{equ:expected_bellman_error}.
\end{proof}

The following lemma shows that policy $\pi$ is not much worse than policy $\pi_g$.
\begin{lemma}\label{lem:eliminate_pi_fi}
For any value that variable $k$ takes in the algorithm, let $\pi$ be the stochastic policy run at Line~\ref{sample:3} of \Eliminate at the $k$-th iteration. Under the event $\mathcal{Z}$, for the first $LC$ times that \Eliminate is called, we have 
\[
\E [V^{\pi}] \ge V^{\pi_g}-(24H+4)\epsilon_k-A\mu_k.
\]
\end{lemma}
\begin{proof} 
At Line~\ref{line:eliminate-3} of \Eliminate at the $k$-th iteration, we have that $\Chk(\{(g\circ_h f,P_k(f))\}_{f\in \mathcal{G}},h,k-2)$ returns \textsc{true} (otherwise the procedure would return at Line~\ref{line:eliminate-4}). By Lemma~\ref{lem:main_check} (\ref{item:lem-main-check-a}) and the event $\mathcal{Z}$ (more specifically, Eq.~\eqref{equ:conc1}), we have 
\[
\left|\sum_{h'=h+1}^{H}\ME(G, h')\right|\le 1.5 H \epsilon_{k-2}=6 H \epsilon_k.
\]
Applying Eq.~\eqref{equ:expected_bellman_error} for $G=\{(g\circ_h f,P_k(f))\}_{f\in \mathcal{G}}$, we get 
\begin{align}
\E_{f\sim P_k}&\E_{x\sim \MD_{\pi_g,h}}\E_{x'\sim p(\cdot\mid x, \pi_f(x))}\left[V^{\pi_f}_{h+1}(x')\right]=\E_{f\sim G}\E_{x'\sim \MD_{\pi_f,h+1}}\left[V^{\pi_f}_{h+1}(x')\right] \nonumber \\ 
&\ge \E_{f\sim G}\E_{x'\sim \MD_{\pi_f,h+1}}\left[f(x',\pi_f(x'))\right]-6H\epsilon_k=\E_{f\sim P_k}\E_{x\sim \MD_{\pi_g,h}}\E_{x'\sim p(\cdot\mid x, \pi_f(x))}\left[f(x',\pi_f(x'))\right]-6H\epsilon_k.\label{equ:cond_1}
\end{align}
Because of the \textsc{pseudo-learn} step at the $(k-1)$-th iteration, and the event $\mathcal{Z}$ (more specifically, Eq.~\eqref{eq:lem-eliminate_eta}), for every $f \in \mathcal{G}$ at the $k$-th iteration, we have 
\begin{equation}\label{equ:cond_2}
\eta(f,g,h)\ge \eta(g,g,h)-(12H+4)\epsilon_k .
\end{equation}
By statement (\ref{item:Q-elim-b}) in the condition $\mathcal{Q}_{\textsc{ELIM}}(g, h, j)$, we have 
\begin{equation}
\left|\E_{x\sim \MD_{\pi_g,h}}\E_{x'\sim p(\cdot\mid x, \pi_{g}(x))}\left[V^{\pi_{g}}_{h+1}(x')\right]-\E_{x\sim \MD_{\pi_g,h}}\E_{x'\sim p(\cdot\mid x, \pi_{g}(x))}\left[g(x',\pi_g(x'))\right]\right|\le 6H\epsilon_j\le 6H\epsilon_k.\label{equ:cond_3}
\end{equation}
Define $\underline{V}^{\pi_{g}}_{h-1}=\E\left[\sum_{h'=1}^{h-1}r(x_{h'},a_{h'})\mid a_{h'}=\pi_g(x_{h'}),x_{h'+1}\sim p(\cdot\mid x_{h'},a_{h'})\right]$, we have
\begin{align*}
\E[V^{\pi}]\ge~ &\underline{V}^{\pi_{g}}_{h-1}+\E_{f\sim P_k}\E_{x\sim \MD_{\pi_g,h}}\E_{x'\sim p(\cdot\mid x, \pi_f(x))}\left[r(x,\pi_f(x))+V^{\pi_f}_{h+1}(x')\right]-A\mu_k\\
\ge~ &\underline{V}^{\pi_{g}}_{h-1}+\E_{f\sim P_k}\E_{x\sim \MD_{\pi_g,h}}\E_{x'\sim p(\cdot\mid x, \pi_f(x))}\left[r(x,\pi_f(x))+f(x',\pi_f(x'))\right]-6H\epsilon_k-A\mu_k\tag{by Eq.~\eqref{equ:cond_1}}\\
= ~& \underline{V}^{\pi_{g}}_{h-1}+\E_{f\sim P_k}\left[\eta(f,g,h)\right]-6H\epsilon_k-A\mu_k\tag{by the definition of $\eta$}\\
\ge~ & \underline{V}^{\pi_{g}}_{h-1}+\eta(g,g,h)-(18H+4)\epsilon_k-A\mu_k\tag{by Eq.~\eqref{equ:cond_2}}\\
\ge~ &\underline{V}^{\pi_{g}}_{h-1}+\E_{x\sim \MD_{\pi_g,h}}\E_{x'\sim p(\cdot\mid x, \pi_{g}(x))}\left[r(x,\pi_{g}(x))+g(x',\pi_{g}(x'))\right]-(18H+4)\epsilon_k-A\mu_k\\
\ge~ &\underline{V}^{\pi_{g}}_{h-1}+\E_{x\sim \MD_{\pi_g,h}}\E_{x'\sim p(\cdot\mid x, \pi_{g}(x))}\left[r(x,\pi_{g}(x))+V^{\pi_{g}}_{h+1}(x')\right]-(24H+4)\epsilon_k-A\mu_k\tag{by Eq.~\eqref{equ:cond_3}}\\
=~ &V^{\pi_{g}}-(24H+4)\epsilon_k-A\mu_k.
\end{align*}
\end{proof}

Combining with statement (\ref{item:Q-elim-a}) in the condition $\mathcal{Q}_{\textsc{ELIM}}(g, h, j)$, we have 
\begin{align} \label{eq:proof-eliminate-main-good-pi}
\E[V^{\pi}] \ge V^*-(24H+4)(H+1)\epsilon_k-A\mu_k.
\end{align}

Therefore, the expected regret incurred at Line~\ref{sample:3} and Line~\ref{sample:4} is upper bounded by,
\begin{align}
\left(\sum_{k=1}^{j}\left((24H+4)(H+1)\epsilon_k+A\mu_k\right)n_{k}^{\text{cb}} \right)  + \left((24H+4)(H+1)\epsilon_j+A\mu_j\right)n_j  \lesssim \frac{MAH^2\ln(P|\MF|/\delta)}{\epsilon_j} . \label{equ:eliminate_regret_part1}
\end{align}

Next, we consider the regret incurred by the call to \Chk at Line~\ref{line:eliminate-3}.  We first verify the condition $\mathcal{P}(G,h,\epsilon_{k-2})$ for every $\Chk(G,h,k-2)$.

Statement (\ref{item:calP-a}) of is a result of Line~\ref{line:eliminate-3} of \Eliminate. Statement (\ref{item:calP-b}) and statement (\ref{item:calP-c}) follow directly from the condition $\mathcal{Q}_{\textsc{ELIM}}(g, h, j)$. By the \textsc{pseudo-learn} step during the $(k-1)$-th iteration, we have 
\[
\tilde{\eta}_{k-1}(f,g,h)\ge \tilde{\eta}_{k-1}(g,g,h)-(6H+1)\epsilon_{k-1}.
\]
By the event $\mathcal{Z}$ (more specifically, Eq.~\eqref{eq:lem-eliminate_eta}), we get 
\[
\eta(f,g,h)\ge \eta(g,g,h)-(6H+2)\epsilon_{k-1}=\eta(g,g,h)-(12+4)\epsilon_k\ge \eta(g,g,h)-(12+4)\epsilon_{k-2},
\]
which establishes statement (\ref{item:calP-d}). Therefore, by Lemma~\ref{lem:main_check} (\ref{item:lem-main-check-c}), the regret incurred by $\Chk(G,h,k-2)$ is upper bounded by
\[
O(H^2\ln^3(|\MF|)\ln(P/\delta)/\epsilon_j).
\]
The overall regret incurred by calling $\Chk$ during $\Eliminate(g,h,j)$ is upper bounded by,
\begin{equation}\label{equ:eliminate_regret_part2}
\sum_{k=1}^{j-2}H^2\ln^3(|\MF|)\ln(P/\delta)/\epsilon_k \lesssim H^2\ln^3(|\MF|)\ln(P/\delta)/\epsilon_j .
\end{equation}

Finally, we analyze the regret incurred by the recursive call to $\Eliminate$ itself at Line~\ref{line:eliminate-4} by establishing condition $\mathcal{Q}_{\textsc{ELIM}}(g', h', k'+1)$ and applying the inductive hypothesis. The following lemma upper bounds the suboptimality of the policy induced by concatenation of value functions. 
\begin{lemma}\label{lem:sub_key}
Let $G=\{(g\circ_h f, P_k(f))\}_{f\in \MF}$ be a distribution of value functions where there exist values $C_1$, $C_2$, and $C_3$, and a layer $h \in [H]$, so that the following conditions are met,
\begin{enumerate}[(a)]
    \item for any $f\in \supp(G)$, $\eta(f,g,h)>\eta(g,g,h)-C_1$,\label{item:key-a}
    \item $\left|\E_{f\sim G}\E_{x'\sim \MD_{f,h+1}}\left[V^{\pi_f}_{h+1}(x')\right]-\E_{f\sim G}\E_{x'\sim \MD_{f,h+1}}\left[f(x',\pi_f(x'))\right]\right|\le C_2,$\label{item:key-b}
    \item $\left|\E_{x'\sim \MD_{g,h+1}}\left[V^{\pi_g}_{h+1}(x')\right]-\E_{x'\sim \MD_{g,h+1}}\left[g(x',\pi_f(x'))\right]\right|\le C_3,$\label{item:key-c}
\end{enumerate}
Then, 
\[
\E_{f\sim G}\left[V^{\pi_f}\right]\ge V^{\pi_g}-(C_1+C_2+C_3).
\]
\end{lemma}
\begin{proof}
The lemma is proved as follows.
\begin{align*}\E_{f\sim G}\left[V^{\pi_f}\right]= ~&\underline{V}^{\pi_{g}}_{h-1}+\E_{f\sim G}\E_{x\sim \MD_{g,h}}\E_{x'\sim p(\cdot\mid x, \pi_f(x))}\left[r(x,\pi_f(x))+V^{\pi_f}_{h+1}(x')\right]\\
\ge ~&\underline{V}^{\pi_{g}}_{h-1}+\E_{f\sim G}\E_{x\sim \MD_{g,h}}\E_{x'\sim p(\cdot\mid x, \pi_f(x))}\left[r(x,\pi_f(x))+f(x',\pi_f(x'))\right]-C_2\tag{by assumption (b)}\\
= ~& \underline{V}^{\pi_{g}}_{h-1}+\E_{f\sim G}\left[\eta(f,g,h)\right]-C_2\tag{by the definition of $\eta$}\\
\ge~ & \underline{V}^{\pi_{g}}_{h-1}+\eta(g,g,h)-(C_1+C_2)\tag{by assumption (a)}\\
\ge~ &\underline{V}^{\pi_{g}}_{h-1}+\E_{x\sim \MD_{g,h}}\E_{x'\sim p(\cdot\mid x, \pi_{g}(x))}\left[r(x,\pi_{g}(x))+g(x',\pi_{g}(x'))\right]-(C_1+C_2)\tag{by the definition of $\eta$}\\
\ge~ &\underline{V}^{\pi_{g}}_{h-1}+\E_{x\sim \MD_{g,h}}\E_{x'\sim p(\cdot\mid x, \pi_{g}(x))}\left[r(x,\pi_{g}(x))+V^{\pi_{g}}_{h+1}(x')\right]-(C_1+C_2+C_3)\tag{by assumption (c)}\\
=~ &V^{\pi_{g}}-(C_1+C_2+C_3).
\end{align*}
\end{proof}

The following lemma is similar to Lemma~\ref{lem:eliminate_pi_fi}.
\begin{lemma}\label{lem:recursion_condition_1}
For any values that variables $g,h,j$ take in the algorithm, suppose $\Eliminate(g,h,j)$ calls $\Eliminate(g',h', k'+1)$ at Line~\ref{line:eliminate-4}. If condition $\mathcal{Q}_{\textsc{ELIM}}(g, h, j)$ holds, then under the event $\mathcal{Z}$, for the first $LC$ times that \Eliminate is called, we have
\begin{equation}\label{equ:recursion_condtion_1}
V^{\pi_{g'}}\ge V^{\pi_g}-(24H+4)\epsilon_{k'+1}.
\end{equation}
\end{lemma}
\begin{proof}
When $\Eliminate(g,h,j)$ recursively calls $\Eliminate(g',h', k'+1)$ during the $k$-th iteration (where $k$ is the loop-variable in Alg.~\ref{alg:eliminate}), we have $g'\in \supp(G)$, where $G=\{(g \circ_h f, P_k(f))\}_{f\in \mathcal{G}}$. By the \textsc{pseudo-learn} step during the $(k-1)$-th iteration, we have that $\tilde{\eta}_{k-1}(f,g,h)\ge \tilde{\eta}_{k-1}(g,g,h)-(6H+1)\epsilon_{k-1}$ for all $f\in \MG.$ By the event $\mathcal{Z}$ (more specifically, Eq.~\eqref{eq:lem-eliminate_eta}), we have that $
\eta(f,g,h)\ge \eta(g,g,h)-(6H+2)\epsilon_{k-1}=\eta(g,g,h)-(12H+4)\epsilon_k,\forall f\in \MG.
$ And by definition of $\eta$ and the fact that $g'\in \supp(G)$ we have $\eta(g',g,h)=\eta(f,g,h)$ for some $f\in \MG.$ Therefore, 
\begin{equation}\label{equ:key_condition_1}
\eta(g',g,h)\ge \eta(g,g,h)-(12H+4)\epsilon_k.
\end{equation}

Since $\Chk(G,h,k-2)$ returns $\textsc{false}(g',h',k')$, it follows from Lemma~\ref{lem:main_check} (\ref{item:lem-main-check-b}) that $\left|\sum_{h''=h+1}^{H}\ME(g',\pi_{g'},h'')\right|\le 3(H-h)\epsilon_{k'}.$ Together with Eq.~\eqref{equ:expected_bellman_error} we have
\begin{equation}\label{equ:key_condition_2}
\left|\E_{x'\sim \MD_{g',h+1}}\left[V^{\pi_{g'}}_{h+1}(x')\right]-\E_{x'\sim \MD_{g',h+1}}\left[g'(x',\pi_{g'}(x'))\right]\right|\le 3(H-h)\epsilon_{k'}\le 6H\epsilon_{k'+1}.
\end{equation}

Note that by Lemma~\ref{lem:main_check} (\ref{item:lem-main-check-b}), we have $h'>h$ and $k'+1<j$. By statement (\ref{item:Q-elim-b}) of the condition $\mathcal{Q}_{\textsc{ELIM}}(g, h, j)$ and Eq.~\eqref{equ:expected_bellman_error}, we have that for the roll-in policy $g$,
\begin{equation}\label{equ:key_condition_3}
\left|\E_{x'\sim \MD_{g,h+1}}\left[V^{\pi_{g}}_{h+1}(x')\right]-\E_{x'\sim \MD_{g,h+1}}\left[g(x',\pi_{g}(x'))\right]\right|\le 6(H-h)\epsilon_{j}\le 6H\epsilon_{k'+1}.
\end{equation}

Consider a singleton distribution $G=\{(g',1)\}$. Eq.~\eqref{equ:key_condition_1}, Eq.~\eqref{equ:key_condition_2} and Eq.~\eqref{equ:key_condition_3} establish the assumptions (a), (b) and (c) in Lemma~\ref{lem:sub_key} respectively, with $C_1=(12H+4)\epsilon_{k'+1}$, and $C_2=C_3=6H\epsilon_{k'+1}.$ Therefore we get,
\[
V^{\pi_{g'}}\ge V^{\pi_g}-(24H+4)\epsilon_{k'+1}.
\]
\end{proof}

 Combining Lemma~\ref{lem:recursion_condition_1} with statement (\ref{item:Q-elim-a}) of $\mathcal{Q}_{\textsc{ELIM}}(g, h, j)$, we have that
\[
V^{\pi_{g'}}\ge V^{\pi_g}-(24H+4)\epsilon_{k'+1}\ge V^*-(24H+4)(h+1)\epsilon_{k'+1}\ge V^*-(24H+4)h'\epsilon_{k'+1},
\]
which establishes statement (\ref{item:Q-elim-a}) of $\mathcal{Q}_{\textsc{ELIM}}(g', h', k'+1)$.

It follows from Lemma~\ref{lem:main_check} that 
\begin{equation}\label{equ:recursion_condition_2}
\left|\sum_{h''=h'+1}^{H}\ME(g',\pi_{g'},h'')\right|\le 3(H-h')\epsilon_{k'}=6(H-h')\epsilon_{k'+1},
\end{equation}
and 
\begin{equation}\label{equ:recursion_condition_3}
\left|\ME(g',\pi_{g'},h')\right|\ge \epsilon_{k'+2}=\epsilon_{k'+1}/2.
\end{equation}
Eq.~\eqref{equ:recursion_condition_2} and Eq.~\eqref{equ:recursion_condition_3} establish statement (\ref{item:Q-elim-b}) and (\ref{item:Q-elim-c}) in $\mathcal{Q}_{\textsc{ELIM}}(g', h', k'+1)$ respectively. 

Now we have established $\mathcal{Q}_{\textsc{ELIM}}(g', h', k'+1)$. Since $k'+1\leq j- 1$, by our induction hypothesis, the regret incurred by the recursive call at Line~\ref{line:eliminate-4} is upper bounded by 
\begin{align}\label{equ:eliminate_regret_part3}
c_{\rm ELIM} MAH^2\ln^3(|\MF|)\ln(P/\delta)/\epsilon_{j-1} .
\end{align}

We combine the regret upper bounds in Eq.~\eqref{equ:eliminate_regret_part1}, Eq.~\eqref{equ:eliminate_regret_part2}, and Eq.~\eqref{equ:eliminate_regret_part3}, and upper bound the overall regret incurred by  $\Eliminate(g,h,j)$ by 
\[
c_{\rm ELIM} MAH^2\ln^3(|\MF|)\ln(P/\delta)/\epsilon_{j} .
\]

\paragraph{Proof of statement (\ref{item:eliminate_main-c}).}
First consider the \textsc{pseudo-learn} step at the $k$-th iteration. 
By statement (\ref{item:Q-elim-b}) of the condition $\mathcal{Q}_{\textsc{ELIM}}(g, h, j)$ we have that $|\sum_{h'=h+1}^{H}\ME(g,\pi_g,h')|\le 6(H-h)\epsilon_j.$ Since $k\le j$, we have that
\begin{align*}
\E_{x\sim \MD_{g,h}}\left[V^{\pi_g}_h(x)\right]=~&\E_{x\sim \MD_{g,h}}\E_{x'\sim p(\cdot\mid x,\pi_{g}(x))}\left[r(x,\pi_g(x))+V^{\pi_g}_{h+1}(x')\right]\\
\ge~ &\E_{x\sim \MD_{g,h}}\E_{x'\sim p(\cdot\mid x,\pi_{g}(x))}\left[r(x,\pi_g(x))+g(x',\pi_g(x'))\right]-6H\epsilon_k\\
=~&\eta(g,g,h)-6H\epsilon_k.
\end{align*}

On the other hand, we have 
\begin{align*}
\E_{x\sim \MD_{g,h}}\left[V^*_h(x)\right]=~&\E_{x\sim \MD_{g,h}}\E_{x'\sim p(\cdot\mid x,\pi^*(x))}\left[r(x,\pi^*(x))+V^*_{h+1}(x')\right]\\
= ~&\E_{x\sim \MD_{g,h}}\E_{x'\sim p(\cdot\mid x,\pi^*(x))}\left[r(x,\pi^*(x))+f^*(x',\pi^*(x'))\right]\\
=~&\eta(f^*,g,h).
\end{align*}

Since $V_h^*(x)\ge V_h^{\pi_{g}}(x)$ for any policy $g$ and state $x$, we have 
\[
\eta(f^*,g,h)=\E_{x\sim \MD_{g,h}}\left[V^*_h(x)\right]\ge \E_{x\sim \MD_{g,h}}\left[V^{\pi_g}_h(x)\right]\ge \eta(g,g,h)-6H\epsilon_k.
\]

By the event $\mathcal{Z}$ (more specifically, Eq.~\eqref{eq:lem-eliminate_eta}), we have $\tilde{\eta}_{k}(f^*,g,h)\ge \tilde{\eta}_{k}(g,g,h)-(6H+1)\epsilon_k$. Therefore $f^*$ will not be eliminated in \textsc{pesudo-learn} step.

As shown earlier, condition $\mathcal{Q}_{\textsc{ELIM}}(g', h', k'+1)$ holds before making the recursive call. Therefore, by our inductive hypothesis,  $f^*$ will not be eliminated by the recursive call at Line~\ref{line:eliminate-4}.

We finally consider the \textsc{learn} step. Note that by the event $\mathcal{Z}$ (more specifically, Eq.~\eqref{equ:conc_bellman_estimate}) and statement (\ref{item:Q-elim-c}) of the condition $\mathcal{Q}_{\textsc{ELIM}}(g, h, j)$, the assumptions in Lemma~\ref{lem:iteration_complexity} are met. Therefore, when \textsc{learn} step is executed, $f^*$ will not be eliminated.

\paragraph{Proof of statement (\ref{item:eliminate_main-a}).} We first consider the case where $\Eliminate(g,h,j)$ recursively calls $\Eliminate(g',h',k'+1)$. Note that after \Eliminate recursively calls itself, this invocation ends immediately. As shown earlier, the condition $\mathcal{Q}_{\textsc{ELIM}}(g', h', k'+1)$ holds at Line~\ref{line:eliminate-4}. Thus, by inductive hypothesis, \textsc{learn} step is executed exactly once, and  Eq.~\eqref{equ:iteration_condition_1} and Eq.~\eqref{equ:iteration_condition_2} hold during the \textsc{learn} step is executed.

On the other hand, if $\Eliminate(g,h,j)$ does not recursively call itself, Eq.~\eqref{equ:iteration_condition_1} and Eq.~\eqref{equ:iteration_condition_2} hold during the \textsc{learn} step because of the event $\mathcal{Z}$ (more specifically, Eq.~\eqref{equ:conc_bellman_estimate}) and statement (\ref{item:Q-elim-c}) in contition $\mathcal{Q}_{\textsc{ELIM}}(g, h, j)$.
\section{Extension to infinite hypothesis space}\label{sec:infinite}

Our algorithm can be extended to infinite hypothesis space by covering argument. First of all, like \textsf{OLIVE}, our algorithm only access to hypothesis $f\in \MG$ via $f(x,\pi_f(x))$ for some state $x$. Thus, the hypothesis space $\MG$ can be equivalently represented by $\Pi\times \MV,$ where $\Pi\subset \MA^{\MX}$ is a set of policy function and $\MV\subset [0,1]^{\MX}$ is a set of value function, representing $\pi_f(x)$ and $f(x,\pi_f(x))$ respectively. Bellman error can be extended to the policy-value hypothesis naturally. For a policy-value function pair $(\pi,v)$ and a roll-in policy $\pi'$, the Bellman error at layer $h$ is defined as 
\begin{equation}
\ME((\pi,v),\pi',h):= \E_{x_h\sim \MD_{\pi',h}}\E_{x_{h+1}\sim p(\cdot\mid x_h,\pi(x_h))}\left[v(x_h)-r(x_h,\pi(x_h))-v(x_{h+1})\right],
\end{equation}
and $\eta((\pi,v),\pi',h)$ is defined as,
\begin{equation}
\eta((\pi,v),\pi',h):= \E_{x_h\sim \MD_{\pi',h}}\E_{x_{h+1}\sim p(\cdot\mid x_h,\pi(x_h))}\left[r(x_h,\pi(x_h))+v(x_{h+1})\right].
\end{equation}

The dependence on the size of hypothesis space comes from two parts: the uniform convergence bound, and the number of binary search steps in \Identify procedure. Thanks to \cite{agarwal2014taming}, the probability distribution $P_k$ found in Line~\ref{line:find_distribution} of \Identify has small support. Lemma~\ref{lem:small_support} shows that $|\supp(P_k)|\le P_{\rm supp}:= 4\ln(1/A\mu_k)/\mu_k$. Therefore, the number of binary steps is bounded by $\lceil \log_2(P_{\rm supp})\rceil$.  We then re-define parameter $\epsilon_l'$ as,
\begin{align*}
\epsilon_{l}'&:=\epsilon_l/\lceil \log_2(P_{\rm supp})+1\rceil.
\end{align*}

To deal with the uniform convergence bound, we assume that the hypothesis $\Pi$ and $\MV$ have finite statistical complexity dimension. Here we use Natarajan dimension and Pseudo dimension as the complexity measurement for function class $\Pi$ and $\MV$ respectively. The definition of Natarajan dimension and Pseudo dimension is given below.

\begin{definition}[Natarajan dimension \cite{natarajan1989learning}]
Let $\MH\subset \MY^{\MX}$ be a hypothesis class. For a set $S\subset \MX$, we say $\MH$ N-shatters $S$ if there exists $h_1,h_2\in \MH$ such that 
\begin{itemize}
	\item $h_1(x)\neq h_2(x),\forall x\in S$, and
	\item $\forall T\subseteq S,$  $\exists h\in \MH$, such that $h(x)=h_1(x),\forall x\in T$ and $h(x)=h_2(x),\forall x\in S\setminus T.$
\end{itemize}
Natarajan dimension $dim_N(\MH)$ is defined as $dim_N(\MH)=\max_{S\subseteq \MX:S\text{ N-shattered by }\MH}|S|.$
\end{definition}

\begin{definition}[Pseudo dimension \cite{haussler1992decision}]
Let $\MH\subset \R^{\MX}$ be a hypothesis class. For a set $S\subset \MX$, we say $\MH$ P-shatters $S$ if there exists $\xi\in \R^{S}$ such that $\forall T\subseteq S, \exists h\in \MH$, such that $\mathbb{I}[h(x)\ge \xi(x)]=\mathbb{I}[x\in T].$ Pesudo dimension $dim_P(\MH)$ is defined as $dim_P(\MH)=\max_{S\subseteq \MX:S\text{ P-shattered by }\MH}|S|.$
\end{definition}

In Appendix \ref{app:extenstion}, we set new values for $n_i^{\text{eval}}$, $n_i$, $n_i^{\text{cb}}$ and $n_i^{\text{id}}$. We then obtain the the following theorem, which is the infinite hypothesis space version of Theorem~\ref{thm:main}, by replacing the uniform convergence statements in the original proof with the ones for low pseudo dimension spaces.

\begin{theorem}\label{thm:main-infinite}
Suppose $dim_N(\Pi) \leq d_{\Pi}$ and $dim_P(\MV) \leq d_{\MV}$. For any $\epsilon>0$ and $\delta>0$, with probability at least $1-\delta$, the overall regret of running \Main for $n$ trajectories with parameter $\epsilon$ is bounded by
\[
{O}\left(M^2AH^4 \ln(P)(\ln^3(P)+6(d_{\Pi}+d_{\MV})\ln(2eA(d_\Pi+d_{\MV}))\ln(P))\ln(P/\delta)/\epsilon+n\epsilon\right)
\]
\end{theorem}
If we choose $\delta = 1/(nH)$ and $\epsilon=\sqrt{\frac{M^2AH^4 \ln^2(P)(\ln^3(P)+6(d_{\Pi}+d_{\MV})\ln(2eA(d_\Pi+d_{\MV}))\ln(P))}{n}}$ for any given $n$, we have the following corollary.
\begin{corollary}\label{cor:main-infinite}
The expected regret of our algorithm for $n$ trajectories is upper bounded by
\[\tilde{O}\left(\sqrt{M^2AH^4n(d_{\Pi}+d_{\MV})}\right), \]
where the $\tilde{O}(\cdot)$ hides poly-logarithmic factors in $M, A, H, \zeta, d_{\Pi}, d_{\MV}$, and $n$. 
\end{corollary}

Proof of Theorem \ref{thm:main-infinite} is deferred to Appendix~\ref{app:extenstion}. Note that the $\ln^3|\MF|$ term in Theorem \ref{thm:main} in the original bound of Theorem~\ref{thm:main} is eliminated because of the following changes.
\begin{itemize}
\item The number of binary search steps is bounded by $\lceil\log_2|P_{\rm supp}|\rceil$. By re-defining the parameter $\epsilon_l'$, we replace term $\ln^3|\MF|$ with $\ln^3(P_{\rm supp})\lesssim \ln^3(P)$. And,
\item The uniform convergence result for low pseudo dimension spaces is used, which replaces a $\ln|\MF|$ term with statistical complexity dimension $d_{\Pi}$ and $d_{\MV}$. 
\end{itemize}

A key technical ingredient in our proof of Theorem~\ref{thm:main-infinite} is a Bernstein-style uniform concentration theorem (namely Lemma~\ref{lem:bernstein_uniform_convergence}) adapted from \cite{massart1986rates}. Lemma~\ref{lem:bernstein_uniform_convergence} is crucial in our analysis for a regret bound that polynomially depends on $M$, $A$ and $H$. It also helps to achieve sharper dependence on $A$. Observe that by Corollary~\ref{cor:main-infinite}, our algorithm can produce an $\epsilon$-optimal with probability at least $0.99$ using $\tilde{O}(M^2 AH^4 (d_{\Pi} + d_{\MV})/\epsilon^2)$ samples. In contrast, the sample complexity of \textsf{OLIVE} analyzed in \cite{Jiang2017ContextualDP} for infinite hypothesis space and constant failure probability is $\tilde{O}(M^2A^2H^3(d_{\Pi}+d_{\MV})/\epsilon^2)$. Our analysis gives a better dependence on $A$, which is due to the help of Lemma~\ref{lem:bernstein_uniform_convergence}. This observation also suggests that Lemma~\ref{lem:bernstein_uniform_convergence} may help to improve the dependence on $A$ in the analysis of \textsf{OLIVE}.

\section{Conclusion}
In this paper we presented \textsf{AVE}, a $\sqrt{n}$-regret algorithm for learning in low-Bellman rank Markov Decision Processes with function approximation. Our algorithm employs sophisticated estimation and elimination techniques, borrows tools from contextual bandit literature, and extends the volumetric argument by \cite{Jiang2017ContextualDP}. We also generalize our algorithm to infinite hypothesis classes, thanks to the proof of a Bernstein-style uniform deviation bound, which also helps to improve the dependence on the action space size compared to \textsf{OLIVE}. For future work, it is worthwhile to design computationally efficient algorithms for learning in MDPs with low Bellman  rank.

\section*{Acknowledgement}
We thank Akshay Krishnamurthy and Zhizhou Ren for valuable discussions.

\bibliography{ref}
\bibliographystyle{unsrt}

\clearpage
\appendix

\section{Low variance estimation}\label{app:low_variance}
Line \ref{line:find_distribution} of \Eliminate finds a distribution that achieves the low variance condition (i.e., Eq.~\eqref{equ:low_variance}). In this section, we show that the distribution $P_k$ exists and can be computed efficiently. This low variance estimation method is adapted from the contextual bandit literature \cite{dudik2011efficient, agarwal2014taming}. The algorithm is described in Algorithm~\ref{alg:distribution}.

\begin{algorithm}[h]
	\caption{$\textsc{Find-Distribution}(g, h, k)$} \label{alg:distribution}
	\label{alg:find_distribution}
	\begin{algorithmic}[1]
		\State Take independent samples and let $\mathcal{H}_{k-1}=\{x^{i}\}_{i=1}^{n_{k-1}^{\text{cb}}}$ where  $x^{i}\sim \MD_{g,h}$.
		\State Find distribution $P_k$ over $\mathcal{G}$, such that $\forall f\in \mathcal{G}$, \label{line:distribution-2}
			\begin{equation}
				\mathop{\E}_{x\sim \mathcal{H}_{k-1}}\left[\frac{1}{(1-A\mu_k)W_{P_k}(x,\pi_f(x))+\mu_k}\right]\le 2A.\label{equ:low_emperical_variance}
			\end{equation}
		\State \textbf{return} $P_k$.
	\end{algorithmic}
\end{algorithm}

Since we do not have access to the distribution $\MD_{g,h}$, we use an empirical estimation $\mathcal{H}_k$ instead. At Line~\ref{line:distribution-2}, we turn to calculate the distribution $P_k$ based on $\mathcal{H}_k$.

The existence of $P_k$ is derives from Sion's Minimax Theorem~\cite{sion1958general}. We have the following lemma regarding Line~\ref{line:distribution-2}.
\begin{lemma}
The set of distributions that satisfies Eq.~\eqref{equ:low_emperical_variance} is non-empty.
\end{lemma}
\begin{proof}
See Corollary 2 of \cite{dudik2011efficient}.
\end{proof}

To compute distribution $P_k$ which satisfies low variance condition, we can use Coordinate Descent Algorithm in \cite{agarwal2014taming}. 
\begin{lemma}\label{lem:small_support}
Distribution $P_k$ that satisfies Eq.~\eqref{equ:low_emperical_variance} can be computed efficiently. Besides, $P_k$ computed by Coordinate Descent Algorithm has support size $|\supp(P_k)|\le 4\ln(1/A\mu_k)/\mu_k.$
\end{lemma}
\begin{proof}
See Theorem 3 of \cite{agarwal2014taming}.
\end{proof}

Now we only need to show that Eq.~\eqref{equ:low_variance} also holds, given that we have Eq.~\eqref{equ:low_emperical_variance}. For an invocation of $\Identify(g,h,j),$ we define 
\begin{align}
V_{P,f}&:= \E_{x\sim \MD_{g,h}}\left[\frac{1}{(1-A\mu_k)W_{P}(x,\pi_f(x))+\mu_k}\right],\\
\hat{V}_{P,f,n}&:= \E_{\{x_1,\cdots,x_{n}\}\sim \MD_{g,h}^{n}}\frac{1}{n}\sum_{i=1}^{n}\left[\frac{1}{(1-A\mu_k)W_{P}(x_i,\pi_f(x_i))+\mu_k}\right].
\end{align}
The following lemma provides a one-sided deviation bound for $\hat{V}_{P,\pi}$.
\begin{lemma}[Lemma 10 of \cite{agarwal2014taming}] \label{lemma:find-distribution-succeed}
Fix any $\mu_k\in [0,1/A]$. For any $\delta\in (0,1)$, if 
$$\mu_k\ge \sqrt{\frac{\ln(2|\MF|/\delta)}{An}},\quad n\ge 4A\ln(2|\MF|/\delta),$$ then with probability at least $1-\delta$, $$V_{P,f}\le 6.4\hat{V}_{P,f,n}+81.3A$$ for all probability distribution $P$ over $\MF$, and all $f\in \MF$.
\end{lemma}

Note that $|\MH_{k-1}|=n_{k-1}^{\text{cb}}=\frac{c_2 A \ln(14L^2C|\MF|/\delta)}{4\epsilon_k^2}$ and $\mu_k=\frac{\epsilon_k}{A}$ for some large enough $c_2.$ By union bound we have, with probability $1-\delta/7$, for the first $LC$ invocations of $\Identify(g,h,j)$,$$V_{P,f}\le 6.4\hat{V}_{P,f,|\MH_{k-1}|}+81.3A$$ for all probability distribution $P$ over $\MF$, all $f\in \MF$ and all $k\in [j].$ Combining with Eq.~\eqref{equ:low_emperical_variance}, we have $V_{P,f}\le 110A$, for the first $LC$ invocations of $\Identify(g,h,j)$, all probability distribution $P$ over $\MF$, all $f\in \MF$ and all $k\in [j].$

\section{High probability events}\label{app:concentration}
In this section we set the parameters for the empirical estimations, and prove the desired events, under which we prove the regret upper bound, happens with high probability. 

The parameters for the empirical estimations are set as follows.
\begin{align}
n_{i}^{\text{eval}}&:=\frac{c_1\ln \left(14L^2C/\delta\right)}{\epsilon_i^2},\\
n_{i}^{\text{cb}}&:=\frac{c_2A\ln \left(14L^2C|\MF|/\delta\right)}{\epsilon_i^2},\\
n_{i}&:=\frac{c_3AM\ln \left(14L^2C|\MF|/\delta\right)}{\epsilon_i ^2},\\
n_{i}^{\text{id}}&:=\frac{c_4\lceil \log_2|\MF|\rceil^2\ln \left(14L^2C/\delta\right)}{\epsilon_i^2},
\end{align}
where $c_i\;(1\le i\le 4)$ are large enough universal constants.

Lemmas \ref{lem:conc_main}, \ref{lem:conc1}, and \ref{lem:conc10} follow directly from Azuma-Hoeffding inequality.

\begin{lemma} [Concentration for Eq.~\eqref{equ:main_concentration}]\label{lem:conc_main}
With probability at least $1-\delta/7,$ for the first $C$ times that Line~\ref{line:main-estimate} of \Main is executed, 
\begin{equation}
\left|\tilde{\ME}_{k}(f,\pi_f,h)-\ME(f,\pi_f,h)\right|\le \epsilon_k/2.\label{equ:conc_main}
\end{equation}
\end{lemma}

\begin{lemma}[Concentration for Eq.~\eqref{equ:check_mix_concentration}] \label{lem:conc1}
With probability at least $1-\delta/7,$  for the first $L^2C$ times that \textsc{Check}$(G,h,j)$ is called,

\begin{equation}\label{equ:conc1}
\left|\hat{\ME}_k(G,h')-\ME(G,h')\right|< \epsilon_k/2,
\end{equation}
for all $h'\in \{h+1,\cdots, H\},k\in [j]$.
\end{lemma}

\begin{lemma} [Concentration for Eq.~\eqref{equ:identify_concentration}] \label{lem:conc10}
With probability at least $1-\delta/7,$  for the first $L^2C$ times that \textsc{Identify}$(G,h,k)$ is called, 
\begin{equation}
\left|\hat{\ME}_l(G_1,h')-\ME(G_1,h')\right|< \epsilon_{l+2}'/2,\label{equ:conc10}
\end{equation}
for all $l\in [k]$, where $\epsilon_{l+2}'=\epsilon_{l+2}/\lceil \log_2|\MF|+1\rceil.$
\end{lemma}

\begin{lemma} [Concentration for Eq.~\eqref{equ:ident_final}] \label{lem:conc_ident_1}
With probability at least $1-\delta/7,$  for the first $L^2C$ times that \textsc{Identify}$(G,h,k)$ is called, 
\begin{equation}
\left|\tilde{\ME}_l(g,\pi_g,h')-\ME(g,\pi_g,h')\right|< \epsilon'_{l+2}/2,\label{equ:conc11}
\end{equation}
for all $l\in [k]$ and $h'\in \{h,h+1,\cdots,H\}.$
\end{lemma}

Lemma \ref{lem:eliminate_eta} and \ref{lem:eliminate_phi} follows from the following Freedman-style inequality.
\begin{theorem}[Freedman-style Inequality, Theorem 13 of \cite{dudik2011efficient}] \label{thm:freedman} Let $y_1,\cdots,y_n$ be a sequence of independent random variables, where $\sum_{i=1}^{n}\E\left[\Var(y_i)\right]\le V$ and $y_i-\E[y_i]\le R$ for all $1\le i\le n$. For any $\delta>0$, if $R<\sqrt{V/\ln(2/\delta)}$, then with probability at least $1-\delta$
$$\left|\sum_{i=1}^{n}y_i-\sum_{i=1}^{n}\E[y_i]\right|\le 2\sqrt{V\ln(2/\delta)}.$$
\end{theorem}

\begin{lemma} [Concentration for Eq.~\eqref{equ:eliminate_eta}] \label{lem:eliminate_eta}
Under the desired event of Lemma~\ref{lemma:find-distribution-succeed}, with probability at least $1-\delta/7,$  for the first $LC$ times that \textsc{Eliminate}$(g,h,j)$ is called, we have
\begin{equation}\label{eq:lem-eliminate_eta}
\left|\tilde{\eta}_k(f,g,h)-\eta(f,g,h)\right|<\epsilon_k/2,
\end{equation}
for all $f\in \MF, 1\le k\le j.$
\end{lemma}
\begin{proof}
Let $y_p=(r^p_h+f(x^p_{h+1},\pi_{f}(x^p_{h+1})))\frac{\mathbb{I}[\pi_{f}(x^p_h)=a^p_h]}{W'_{P_k}(x_h^P,a_h^P)}$ for $1\le p \le n_k^{\text{cb}}$. Then, we have $y_t-\E[y_t]\le 2/\mu_{k}.$ On the other hand, $$\E[\Var(y_p)]\le \E[y_p^2]\le 4\E\left[\frac{\mathbb{I}[\pi_{f}(x^p_h)=a^p_h]}{W'_{P_k}(x_h^P,a_h^P)^2}\right]\le \E\left[\frac{4}{W'_{P_k}(x_h^P,\pi_f(x_h))}\right].$$
Combining with Eq.~\eqref{equ:low_variance} (which holds because of the desired event of Lemma~\ref{lemma:find-distribution-succeed}), we have $\sum_{p=1}^{n_k^{\text{cb}}}\E[\Var(y_p)]\le 440n_k^{\text{cb}}A.$ By the definition of $\mu_k$ and $n_{k}^{\text{cb}}$, we have $$\frac{2}{\mu_k}=\frac{A}{\epsilon_k}=A\sqrt{\frac{n_k^{\text{cb}}}{c_2A\ln(2C|\MF|/\delta)}}\le \sqrt{440n_k^{\text{cb}}A/\ln(2LC|\MF|/\delta)}.$$ Applying Theorem \ref{thm:freedman} we have $$\left|\frac{1}{n_k^{\text{cb}}}\sum_{i=1}^{n_k^{\text{cb}}}y_i-\eta(f,g,h)\right|\le 2\sqrt{440A\ln(2LC|\MF|/\delta)/n_k^{\text{cb}}}$$ with probability at least $1-\delta/(LC|\MF|).$ The result follows from setting $c_2=14080$, and applying union bound for all $f\in \MF$, $k\in [j]$ and the first $C$ invocations.
\end{proof}
Similarly, we have the following concentration result for Eq.~\eqref{equ:bellman_estimate}.
\begin{lemma} [Concentration for Eq.~\eqref{equ:bellman_estimate}] \label{lem:eliminate_phi}
Under the desired event of Lemma~\ref{lemma:find-distribution-succeed}, with probability at least $1-\delta/7$,  for the first $LC$ times that \textsc{Eliminate}$(g,h,j)$ is called, we have
\begin{equation}\label{equ:conc_bellman_estimate}
\left|\tilde{\ME}_j(f,\pi_{g},h')-\ME(f,\pi_{g},h')\right|<\phi_j,
\end{equation}
for all $f\in \MF$.
\end{lemma}

We now define $\mathcal{Z}$ to be the conjunction of the desired events in Lemmas~\ref{lemma:find-distribution-succeed}, \ref{lem:conc_main}, \ref{lem:conc1}, \ref{lem:conc10}, \ref{lem:conc_ident_1}, \ref{lem:eliminate_eta}, and \ref{lem:eliminate_phi}. We have 
\begin{align} \label{eq:calM-high-probability}
\Pr[\mathcal{Z}]\ge 1-\delta.
\end{align}

\section{The volumetric argument}\label{app:volume}
In this section we prove Lemma~\ref{lem:iteration_complexity}. We use the volumetric argument that is adapted from \cite{Jiang2017ContextualDP}. For readers who are not familiar with \textsf{OLIVE} algorithm, we first give a high-level idea of the volumetric argument. To help better understanding the underlying principle, we first ignore approximation error cause by finite sampling. That is, we assume the algorithm has access to the value $\ME(f,\pi_{g},h)$. In this case, step~\ref{item:olive-elimiate} of \textsc{OLIVE} algorithm (see Section~\ref{sec:prelim}) can also ignore approximation error:
\begin{itemize}
	\item[3'.] Remove all $f\in\mathcal G$ from $G$ with $|\ME(f,\pi_g,h)|> 0.$\label{item:olive-eliminate-toy}
\end{itemize}

Recall that Bellman factorization gives $$\ME(f,\pi_g,h)=\left<\nu_h(g),\xi_h(f)\right>.$$ Consider the set $V=\{\xi_h(f):f\in \MG\}.$ For any $g\in \MG$, Elimination criteria of \textsc{OLIVE} remove all $f$ such that $\left<\nu_h(g),\xi_h(f)\right>\neq 0.$ Therefore the rank of set $V$ is reduced by 1 whenever elimination step of \textsc{OLIVE} is executed. As a result, the number of elimination steps is bounded by the Bellman rank $M$.

When approximation error is considered, the elimination step also shrinks the set $V$ significantly. However, the linear algebraic dimension is not stable with respect to errors. Instead, volume of the minimal covering ellipsoid of $V$ is used as a complexity measurement. Similar to the ellipsoid method for solving linear programming, every elimination step reduces the volume of $V$ significantly. In fact, the number of elimination steps is still bounded by the Bellman rank $M$ upto logarithmic factors.

The proof of Lemma~\ref{lem:iteration_complexity} is similar to that of \textsc{OLIVE} algorithm, except that \Eliminate algorithm is done in multiple level. For every error level $j$ and layer $h$, we use the volumetric argument respectively. Then the overall number of elimination steps of \textsc{AVE} algorithm is also bounded. Below we present the rigorous proof.

\begin{proof}[Proof of Lemma~\ref{lem:iteration_complexity}]
Let us consider a fixed pair of $j$ and $h$.

Recall that the Bellman factorization implies that 
\[
\ME(f,\pi_g,h)=\left<\nu_h(g),\xi_h(f)\right>,\quad \forall f,g\in \MF, h\in [H]\]
where $\|\nu_v(g)\|_2\|\xi_h(f)\|_2\le \zeta<\infty$.

Let $\mathcal{G}_i$ be the hypothesis space after $i$-th execution of \textsc{learn} step with layer $h$ and precision $\phi_j$. Let $\mathcal{G}_0 = \mathcal{F}$ be the original hypothesis space. We define $V_i=\{\xi_h(f):f\in \mathcal{G}_i\}$ and $B_i$ be minimum covering ellipsoid of $V_i$. For any $i$, suppose the $i$-th \textsc{learn} step with layer $h$ and precision $\phi_j$ is executed in an invocation of \Eliminate with parameters $(g,h,j)$. Let $p_i=\nu_h(g)$. We will show that there exists $v\in V_{i-1}$ such that $|p_i^{\top}v|>\epsilon_j/2$.

The existence is trivial if $g\in \mathcal{G}_{i-1}$. Otherwise, the current \Eliminate must be recursively invoked by the \Eliminate procedure with parameter $(g',h',j')$, where $h'<h$ and $g=g'\circ_{h'}f$ for some $f\in \mathcal{G}_{i-1}$. Consider the vector $v=\xi_h(f)$, where $f\in \mathcal{G}_{i-1}$ implies that $v\in V_{i-1}$. Note that by Assumption~\ref{asmp:closedness}, despite the fact that $g$ is the concatenation of functions in $\MF$, we still have $g\in \MF.$ Therefore, we have the following Bellman factorization $\left|p_i^{\top}v\right|=\left|\ME(f,\pi_g,h)\right|.$ By $g=g'\circ_{h'}f$ and Eq.~\eqref{equ:iteration_condition_2}, we have 
\[
\left|p_i^{\top}v\right|=\left|\ME(f,\pi_g,h)\right|=\left|\ME(g,\pi_g,h)\right|\ge \epsilon_j/2.
\]

Let $V_+=\{v\in B_{i-1}:|p_i^{\top}v|<2\phi_j\}$ and $B_+$ the minimum covering ellipsoid of $V_+$. Then by the elimination criteria (Line~\ref{line:bellman_elimination} of Alg.~\ref{alg:eliminate}) and Eq.~\eqref{equ:bellman_estimate} we have $V_i\subseteq V_+$, which implies that $V_i\subseteq B_+$. Since $B_i$ is the minimum covering ellipsoid, $\vol(B_i)\le \vol(B_+).$ Since $2\phi_j/(\epsilon_j/2)=1/2\sqrt{M},$ by Corollary~\ref{cor:iteration} we have $\vol(B_+)\le 0.6\vol(B_{i-1}).$ Therefore, if the \textsc{learn} step is executed for $t$ times with layer $h$ and precision $\phi_j$, we have that $\vol(B_t)\le 0.6^t\vol(B_0).$

Let $L_{\xi}=\sup_{f\in \MF}\|\xi_h(f)\|_2$ and $L_{\nu}=\sup_{f\in \MF}\|\nu_h(f)\|_2.$ Then we have $V_0\subseteq \{v\in \R^M:\|v\|\le L_{\xi}\},$ which implies $\vol(B_0)\le L_{\xi}^{M}\mathfrak{B}_{M}$, where $\mathfrak{B}_{M}$ denotes the volume of a unit ball in $\R^{M}.$ On the other hand, we have $\{v\in \R^{M}:\|v\|\le 0.5\phi_j/ L_{\nu}\}\subseteq V_{t}$, which means that $\vol(B_t)\ge \left(1/(2\phi_j L_{\nu})\right)^M\mathfrak{B}_{M}.$ Therefore, by basic algebra we get $$t\le \log_{5/3}\left(\frac{\vol(B_0)}{\vol(B_t)}\right)\le M\log_{5/3}(L_{\xi}L_{\nu}/(2\phi_j)).$$

The lemma then follows because of Assumption~\ref{assumption:bellman_rank} which states that $L_{\xi}L_{\nu}\le \zeta.$
\end{proof}

In the remaining part of this section, we present technique tools used in the proof above. The following result is an adaption from the work of \cite{todd1982minimum}.
\begin{lemma}[Lemma 11 of \cite{Jiang2017ContextualDP}]
Let $V$ be an closed and bounded subset of $\R^{d}$, let $B=\{v\in \R^{d}:\|Mv\|\le 1\}$ be an ellipsoid containing $V$. Suppose there exists $v\in V$ such that $|p^\top v|\ge \kappa.$ Define $B_+$ to be the minimum covering ellipsoid of set $\{v\in B:|p^{\top}v|\le \gamma\}.$ If $\gamma/\kappa\le 1/\sqrt{d},$ then
$$\frac{\vol(B)}{\vol(B_+)}\le \sqrt{d}\frac{\gamma}{\kappa}\left(\frac{d}{d-1}\right)^{(d-1)/2}\left(1-\frac{\gamma^2}{\kappa^2}\right)^{(d-1)/2}.$$
\end{lemma}
\begin{corollary}[Fact 4 of \cite{Jiang2017ContextualDP}]\label{cor:iteration}
When $\gamma/\kappa=1/3\sqrt{d}$, we have that 
\[
\frac{\vol(B)}{\vol(B_+)}\le 3/5.
\]
\end{corollary}

\section{Omitted proofs in Section \ref{sec:analysis}}\label{app:analysis}

\subsection{Proof of Lemma~\ref{lem:main_identify}}
We condition on the event $\mathcal{Z}$ throughout the proof.

\paragraph{Proof of statement (\ref{item:identify-a}).}
We prove by induction on $\textsc{step}$ that, whenever at Line~\ref{line:identify_while} of \Identify (i.e., before the condition that $|\supp(G)|>1$ is checked), we have that 
\begin{align}\label{eq:proof-identify-a-1}
\left|\sum_{h'=h+1}^{H}\ME(G,h')\right|\ge (H-h)\left(\epsilon_{k+1}-\textsc{step}\cdot \epsilon'_{k+2}\right).
\end{align}

The base case is that when $\textsc{step}=0$. The condition $\mathcal{Q}_{\textsc{ID}}(G, h, k)$ implies that 
\[
\left|\sum_{h'=h+1}^{H}\ME_{k}(G,h')\right|>(H-h)\epsilon_{k+1}.
\]


Now suppose Eq.~\eqref{eq:proof-identify-a-1} is true for $\textsc{step}=s$. We prove the same equation for $\textsc{step}=s+1$. First consider the case when  $c=\textsc{true}$ at Line~\ref{line:identify_while} with $\textsc{step} = s + 1$. In this case, Line~\ref{line:case_1} was executed in the previous iteration of the \textbf{while}-loop. The \textbf{if}-condition at Line~\ref{line:identify_if_1} implies that after executing Line~\ref{line:case_1}, we have
\[
\left|\sum_{h'=h+1}^{H}\hat{\ME}_{k}(G,h')\right|>(H-h)\left(\epsilon_{k+1}-((s + 1) -0.5)\epsilon'_{k+2}\right).
\]
By the event $\mathcal{Z}$ (more specifically, Eq.~\eqref{equ:conc10}),  we have that
\[
\left|\sum_{h'=h+1}^{H}\ME(G,h')\right|>(H-h)\left(\epsilon_{k+1}-(s+1) \epsilon'_{k+2}\right).
\]

Now consider the case when $c=\textsc{false}$ at Line~\ref{line:identify_while} with $\textsc{step} = s + 1$. In this case the \textbf{if}-condition did not hold for iteration $l=k$ in the previous \textbf{while}-loop, which implies that 
\[
\left|\sum_{h'=h+1}^{H}\hat{\ME}_{k}(G_1,h')\right|\le (H-h)\left(\epsilon_{k+1}-((s + 1)-0.5)\cdot \epsilon'_{k+2}\right).
\]
By the event $\mathcal{Z}$ (more specifically, Eq.~\eqref{equ:conc10}), we have that 
\begin{align}\label{eq:proof-identify-a-2}
\left|\sum_{h'=h+1}^{H}\ME(G_1,h')\right|\le (H-h)\left(\epsilon_{k+1}-s\cdot \epsilon'_{k+2}\right).
\end{align}

Let $\MF_1=\MF'$ and $\MF_2=\supp(G)\setminus \MF'$. It follows from the definition of $\ME(G,h')$ that for all $h'\in \{h+1,h+2,\cdots,H\}$,
\begin{equation}\ME(G,h')=G(\MF_1)\cdot \ME(G_1,h')+G(\MF_2) \cdot \ME(G_2,h'),\label{equ:average_bellman}\end{equation} where $G(\MF_i)=\Pr_{f \sim G} [f \in \MF_i]$, and we have $G(\MF_1)+G(\MF_2)=1$. Eq.~\eqref{equ:average_bellman} further implies that
\begin{align}\label{eq:proof-identify-a-4}
\left|\sum_{h'=h+1}^{H}\ME(G,h')\right| \leq (1 - G(\mathcal{F}_2)) \cdot \left|\sum_{h'=h+1}^{H}\ME(G_1,h')\right| + G(\mathcal{F}_2) \cdot \left|\sum_{h'=h+1}^{H}\ME(G_2,h')\right| .
\end{align}
Our induction hypothesis implies that
\begin{align}\label{eq:proof-identify-a-3}
\left|\sum_{h'=h+1}^{H}\ME(G,h')\right|\ge (H-h)\left(\epsilon_{k+1}-s \cdot \epsilon_{k+2}'\right).
\end{align}
Combining Eq.~\eqref{eq:proof-identify-a-2},  Eq.~\eqref{eq:proof-identify-a-3}, and  Eq.~\eqref{eq:proof-identify-a-4}, we have that
\[
\left|\sum_{h'=h+1}^{H}\ME(G_2,h')\right|\ge (H-h)\left(\epsilon_{k+1}- s \cdot \epsilon_{k+2}'\right)>(H-h)\left(\epsilon_{k+1}-(s + 1) \cdot \epsilon_{k+2}'\right),
\]
which establishes Eq.~\eqref{eq:proof-identify-a-1} since we let $G\gets G_2$.

Since the size of $\supp(G)$ is halved in every iteration, the \text{while}-loop terminates after at most $\lceil \log_2 |\MF|\rceil$ iterations. By Eq.~\eqref{eq:proof-identify-a-1}, when the \textbf{while}-loop ends, we have that 
\begin{equation}\label{equ:identify_induction}
\left|\sum_{h'=h+1}^{H}\ME(g,\pi_g,h')\right|=\left|\sum_{h'=h+1}^{H}\ME(G,h')\right|\ge (H-h)\left(\epsilon_{k+1}-\lceil \log_2 |\MF|\rceil\cdot \epsilon'_{k+2}\right)>(H-h)\left(\epsilon_{k+2}+\epsilon_{k+2}'\right).
\end{equation}

Now we prove that \Identify must return at Line~\ref{line:identify_return}. Eq.~\eqref{equ:identify_induction} implies that 
\[
\max_{h':h+1\le h'\le H}\left|\ME(g,\pi_g,h')\right|\ge \epsilon_{k+2}+\epsilon_{k+2}'.
\] 
It follows from the event $\mathcal{Z}$ (more specifically, Eq.~\eqref{equ:conc11}) that, 
\[
\max_{h':h+1\le h'\le H}\left|\tilde{\ME}_{k}(g,\pi_g,h')\right|\ge \epsilon_{k+2}+0.5\epsilon_{k+2}',
\]
which satisfies the \textbf{if}-condition at Line~\ref{line:identify_return_if} for the iteration $l=k$. The remaining claims of statement (\ref{item:identify-a}) directly follow from the description of the algorithm.

\paragraph{Proof of statement (\ref{item:identify-b}).}
When \Identify returns, the \textbf{if}-condition at Line~\ref{line:identify_return_if} is satisfied. I.e., 
\[
\left|\tilde{\ME}_{l}(g_r,\pi_{g_r},h_r)\right|>\epsilon_{l+2}+0.5\epsilon'_{l+2}.
\]
By  the event $\mathcal{Z}$ (more specifically, Eq.~\eqref{equ:conc11}), we have that
\[
\left|\ME(g_r,\pi_{g_r},h_r)\right|>\epsilon_{l+2},
\]
which proves statement (\ref{item:identify-b}).

\paragraph{Proof of statement (\ref{item:identify-c}).}
Note that when \Identify returns at iteration $l$, the \textbf{if}-condition at Line~\ref{line:identify_return_if} did not hold for iteration $(l-1)$. Thus, $\left|\tilde{\ME}_{l-1}(g,\pi_g,h')\right|\le \epsilon_{l+1}+0.5\epsilon_{l+1}'$ for all $h'\in \{h+1,h+2,\cdots, H\}.$ Combining with  the event $\mathcal{Z}$ (more specifically, Eq.~\eqref{equ:conc11}), we have that
\[
\left|\ME(g,\pi_g,h')\right|\le \left(\epsilon_{l+1}+\epsilon_{l+1}'\right)\le 1.5\epsilon_{l+1}=3\epsilon_l,\forall h'\in \{h+1,h+2,\cdots, H\}.
\]

\paragraph{Proof of statement (\ref{item:identify-d}).}
First we focus on the regret incurred by the \textbf{while}-loop starting from Line~\ref{line:identify_while}. Let $\pi = \pi_f$ where $f \sim G_1$ at Line~\ref{sample:6} of $\Identify(G,h,k)$. Since $\supp(G_1)\subseteq\supp(G)$, statement (\ref{item:identify-d}) in the condition $\mathcal{P}(G, h, \epsilon_k)$ implies that
\begin{equation}\label{equ:ident:key_condition_1}
\eta(f,g,h)\ge \eta(g,g,h)-(12H+4)\epsilon_k, \forall f\in \supp(G_1).
\end{equation}
Statement (\ref{item:calP-c}) in the condition $\mathcal{P}(G, h, \epsilon_k)$ implies that
\begin{equation}\label{equ:ident:key_condition_3}
\left|\E_{x'\sim \MD_{g,h+1}}\left[V^{\pi_g}_{h+1}(x')\right]-\E_{x'\sim \MD_{g,h+1}}\left[g(x',\pi_f(x'))\right]\right|\le 6H\epsilon_k,
\end{equation}

If the algorithm proceeds to iteration $l$, the \textbf{if}-condition at Line~\ref{line:identify_if_1} did not hold at iteration $(l-1)$, which means that $\left|\sum_{h'=h+1}^{H}\hat{\ME}_{l-1}(G_1,h)\right|\le (H-h)\left(\epsilon_{l}-(\textsc{step}-0.5)\cdot \epsilon'_{l+1}\right)\le (H-h)\epsilon_{l}.$ By the event $\mathcal{Z}$ (more specifically, Eq.~\eqref{equ:conc10}) and Eq.~\eqref{equ:expected_bellman_error}, we have that
\begin{equation}\label{equ:ident:key_condition_2}
\left|\E_{f\sim G_1}\E_{x'\sim \MD_{f,h+1}}\left[V^{\pi_f}_{h+1}(x')\right]-\E_{f\sim G_1}\E_{x'\sim \MD_{f,h+1}}\left[f(x',\pi_f(x'))\right]\right|\le 6H\epsilon_{l}.
\end{equation}

Note that $l\le k$. For the distribution $G_1$, Eq.~\eqref{equ:ident:key_condition_1}, \eqref{equ:ident:key_condition_2}, and Eq.~\eqref{equ:ident:key_condition_3} establish statement (\ref{item:key-a}), (\ref{item:key-b}), and (\ref{item:key-c}) in Lemma~\ref{lem:sub_key} respectively, where $C_1=(12H+4)\epsilon_{l}, C_2=C_3=6H\epsilon_{l}.$ Therefore we get,
\[
\E [V^{\pi}] \ge V^{\pi_g}-(24H+4)\epsilon_{l}.
\]
Combining with statement (\ref{item:calP-b}) in the condition $\mathcal{P}(G, h, \epsilon_k)$, we deduce that $\E[V^{\pi}] \ge V^*-(24H+4)H\epsilon_l$. Therefore, the expected regret incurred by a single iteration of the \textbf{while}-loop starting from Line~\ref{line:identify_while} is upper bounded by
\[
\sum_{l=1}^{k}n_{l}^{\text{id}}(24H+4)H\epsilon_l\lesssim H^2\ln^2|\MF|\ln(P/\delta)/\epsilon_k.
\]
As shown before, the \textbf{while}-loop terminates after $\lceil\log_2 |\MF|\rceil$ steps. Thus the total regret incurred by the \textbf{while}-loop is upper bounded by $O\left(H^2\ln^3|\MF|\ln(P/\delta)/\epsilon_k\right).$

Similarly, we can upper bound the regret for policy $\pi_{g_r}$ by
$O\left(H^2\ln^2|\MF|\ln(P/\delta)/\epsilon_k\right)$. Combining the two parts together, we prove statement (\ref{item:identify-d}) for Lemma~\ref{lem:main_identify}.

\subsection{Proof of Lemma~\ref{lem:main_check}}
We condition on the event $\mathcal{Z}$ throughout the proof.

\paragraph{Proof of statement (\ref{item:lem-main-check-a}).}
When $\Chk(G,h,j)$ returns \textsc{true}, for the iteration $k=j$ we have 
\[
\left|\sum_{h'=h+1}^{H}\hat{\ME}_{k}(G,h')\right|\le (H-h)\epsilon_j.
\]
Therefore, together with the event $\mathcal{Z}$ (more specifically, Eq.~\eqref{equ:conc1}), we have that
\[
\left|\sum_{h'=h+1}^{H}\ME(G,h')\right|\le 1.5(H-h)\epsilon_j.
\]

\paragraph{Proof of statement (\ref{item:lem-main-check-b}).}
When $\Chk(G,h,j)$ returns $(\textsc{false},g_r,h_r,k_r)$ during the $k$-th iteration, the tuple $(g_r,h_r,k_r)$ is returned from $\Identify(G,h,k)$. By the \textbf{if}-condition at Line~\ref{line:check_if} we get
\[
\left|\sum_{h'=h+1}^{H}\hat{\ME}_{k}(G,h')\right|\ge \epsilon_k.
\]
Combining with the event $\mathcal{Z}$ (more specifically, Eq.~\eqref{equ:conc1}), we have that
\[
\left|\sum_{h'=h+1}^{H}\ME(G,h')\right|\ge 0.5\epsilon_k=\epsilon_{k+1},
\]
which verifies the condition $\mathcal{Q}_{\textsc{ID}}(G, h, k).$ Then, statement (\ref{item:lem-main-check-b}) follows from Lemma~\ref{lem:main_identify}.

\paragraph{Proof of statement (\ref{item:lem-main-check-c}).}
As shown above, when $\Chk(G,h,j)$ calls $\Identify(G,h,k)$, the condition $\mathcal{Q}_{\textsc{ID}}(G, h, k)$ holds. Note that $k\le j$. By Lemma~\ref{lem:main_identify} (\ref{item:identify-d}), the regret incurred by calling \Identify is upper bounded by 
\begin{align} \label{eq:lem-main-check-c-1}
O(H^2\ln^3(|\MF|)\ln(P/\delta)/\epsilon_j) .
\end{align}

We now focus on the regret incurred by Line~\ref{line:check_mix_run} during each iteration. Suppose we are at the $k$-th iteration. Let $\pi = \pi_{f}$ where $f \sim G$.  By statement (\ref{item:calP-a}) in the condition $\mathcal{P}(G, h, \epsilon_j)$, $G$ has the form $G=\{(g\circ_h f, P(f))\}_{f\in \MF}$ for some distribution $P$ and policy $g$. Statement (\ref{item:calP-d}) in the condition $\mathcal{P}(G, h, \epsilon_j)$ implies that
\begin{equation}\label{equ:check:key_condition_1}
\eta(f,g,h)\ge \eta(g,g,h)-(12H+4)\epsilon_j, \forall f\in \supp(G).
\end{equation}
Statement (\ref{item:calP-c}) in the condition $\mathcal{P}(G, h, \epsilon_j)$ implies that
\begin{equation}\label{equ:check:key_condition_3}
\left|\E_{x'\sim \MD_{g,h+1}}\left[V^{\pi_g}_{h+1}(x')\right]-\E_{x'\sim \MD_{g,h+1}}\left[g(x',\pi_f(x'))\right]\right|\le 6H\epsilon_j,
\end{equation}

Since \textbf{if}-condition at iteration $(k-1)$ was not met, we have that
\[
\left|\sum_{h'=h+1}^{H}\hat{\ME}_{k}(G,h')\right|\le (H-h)\epsilon_{k-1}.
\]
Combining with the event $\mathcal{Z}$ (more specifically, Eq.~\eqref{equ:conc1}), we have that
\[
\left|\sum_{h'=h+1}^{H}\ME(G,h')\right|\le 1.5(H-h)\epsilon_{k-1}.
\]

It follows from Lemma~\ref{lem:expected_bellman_error} that
\begin{equation}\label{equ:check:key_condition_2}
\left|\E_{f\sim G}\E_{x'\sim \MD_{f,h+1}}\left[V^{\pi_f}_{h+1}(x')\right]-\E_{f\sim G}\E_{x'\sim \MD_{f,h+1}}\left[f(x',\pi_f(x'))\right]\right|\le 1.5(H-h)\epsilon_{k-1}\le 3(H-h)\epsilon_{k}.
\end{equation}

Note that $k\le j$. The assumptions for Lemma~\ref{lem:sub_key} follow from Eq.~\eqref{equ:check:key_condition_1}, Eq.~\eqref{equ:check:key_condition_2}, and Eq.~\eqref{equ:check:key_condition_3} respectively, with $C_1=(12H+4)\epsilon_k, C_2=3(H-h)\epsilon_k$ and $C_3=6H\epsilon_k.$ Therefore, we have $\E [V^{\pi}] \ge V^{\pi_g}-(21H+4)\epsilon_k$. Combining with statement (\ref{item:calP-b}) in the condition $\mathcal{P}(G, h, \epsilon_j)$, we have that
\[
\E[V^{\pi}] \ge V^{*}-(24H+4)(H+1)\epsilon_k.
\]

Therefore, the expected regret incurred by Line~\ref{line:check_mix_run} during iteration $k$ is upper bounded by 
\[
n_{k}^{\text{eval}}(24H+4)(H+1)\epsilon_k\lesssim H^2\ln(P/\delta)/\epsilon_k.
\]
It follows that the overall regret introduced Line~\ref{line:check_mix_run} is upper bounded by
\begin{align} \label{eq:lem-main-check-c-2}
O\left(\sum_{k=1}^{j}H^2\ln(P/\delta)/\epsilon_k\right) \lesssim H^2\ln(P/\delta)/\epsilon_k,
\end{align}
Statement (\ref{item:lem-main-check-c}) follows by combining the regret bounds in Eq.~\eqref{eq:lem-main-check-c-1} and Eq.~\eqref{eq:lem-main-check-c-2}.

\section{Extension to infinite hypothesis space}\label{app:extenstion}
Suppose the hypothesis $\Pi$ has finite Natarajan dimension $d_\Pi$, and $\MV$ has finite 
Pseudo dimension $d_{\MV}.$ Then, the parameter of our algorithm is set as following, where $d:= 6(d_{\Pi}+d_{\MV})\ln(2eA(d_\Pi+d_{\MV})).$
\begin{align*}
n_{i}^{\text{eval}}&:=\frac{c_1\ln \left(14L^2C/\delta\right)}{\epsilon_i^2},\\
n_{i}^{\text{cb}}&:=\frac{c_5A(d\ln(A/\epsilon_i)+\ln \left(140L^2C/\delta\right))}{\epsilon_i^2},\\
n_{i}&:=\frac{c_6AM(d(\ln(A/\epsilon_i)+\ln \left(140L^2C/\delta\right))}{\epsilon_i ^2},\\
n_{i}^{\text{id}}&:=\frac{c_7\lceil \log_2(4\ln(1/A\mu_k)/\mu_k)\rceil^2\ln \left(14L^2C/\delta\right)}{\epsilon_i^2},
\end{align*}
where $c_i\;(5\le i\le 7)$ are large enough universal constants.

\subsection{Uniform convergence}
In this section, we show that the high probability event $\MZ$ holds with probability at least $1-\delta.$

Note that Lemmas~\ref{lem:conc_main}, \ref{lem:conc1}, \ref{lem:conc10}, \ref{lem:conc_ident_1} holds for infinite hypothesis without modification. Now we present the proof of Lemma~\ref{lem:eliminate_eta} for the infinite hypothesis setting. The followings are standard results in statistical learning literature.

\begin{definition}[Covering number] For a hypothesis class $\MH:\MX\to \R$, and any $\epsilon>0$, we say a set $\MC\subseteq \MH$ is a \emph{proper} $\epsilon$-covering set for $X=\{x_1,x_2,\cdots,x_n\}\in \MX^n,$ for any $g\in \MH$, there exists $g'\in \MC$, such that $\frac{1}{n}\sum_{i=1}^{n}|g(x_i)-g'(x_i)|\le \epsilon.$ The covering number is defined as,$$\MN(\epsilon,\MH,X):= \min\nolimits_{\MC:\MC\text{ is a proper $\epsilon$-covering set for }X}|\MC|.$$
\end{definition}
We also define $\MN(\epsilon,\MH,n):=\max_{X\in \MX^n}\MN(\epsilon,\MH,X).$

\begin{lemma}[Covering number for hypothesis with finite pseudo dimension \cite{Haussler1995SpherePN}] For a hypothesis $\MH\subseteq \R^{\MX}$ with pseudo dimension $d$, we have
$$\MN(\epsilon,\MH,n)\le e(d+1)\left(\frac{2e}{\epsilon}\right)^d\le \left(\frac{4e^2}{\epsilon}\right)^d.$$
\end{lemma}

\begin{lemma}[Lemma 21 of \cite{Jiang2017ContextualDP}]\label{lem:lemma21} Let $\MY$ be label space with $|\MY|=A.$ Suppose hypothesis $\Pi\subseteq \MY^{\MX}$ has Natarajan dimension $d_\Pi$ and $\MV\subseteq [0,1]^{\MX}$ has Pseudo dimension $d_{\MV},$ where $d_{\Pi}\ge 6$ and $d_{\MV}\ge 6.$ Then the hypothesis $\MH=\{(x,a,x')\to \mathbb{I}[a=\pi(x)]v(x'):\pi\in \Pi,v\in \MV\}$ has pseudo dimension $dim_P(\MH)\le 6(d_{\Pi}+d_{\MV})\ln(2eA(d_\Pi+d_{\MV}).$
\end{lemma}

\begin{proof}(Lemma~\ref{lem:eliminate_eta} for infinite hypothesis setting)\textbf{. }Consider when $\Eliminate(g,h,j)$ is at at iteration $k$. Note that given distribution $W'_{P_k}(x,a)$, $\tilde{\eta}$ is defined as $$\tilde{\eta}_k((\pi,v),\pi',h)=\frac{1}{n_k^{\text{cb}}}\sum_{p=1}^{n_k^{\text{cb}}}(r^p_h+v(x^p_{h+1}))\frac{\mathbb{I}[\pi(x^p_h)=a^p_h]}{W'_{P_k}(x_h^P,a_h^P)},$$ which, by importance sampling, is an unbiased empirical estimation of 
$$\eta((\pi,v),\pi',h)=\E_{x\sim \MD_{\pi',h}}\E_{x'\sim p(\cdot\mid x,\pi(x))}\left[\left(r(x,a)+v(x')\right)\right].$$
Define $$g_{\pi,v}(x,a,x'):= v(x')\frac{\mathbb{I}[\pi(x)=a]}{W'_{P_k}(x,a)}.$$ By Lemma~\ref{lem:lemma21}, hypothesis $\MH=\{g_{\pi,v}:\pi\in \Pi,v\in \MV\}$ has pseudo dimension at most $$dim_P(\MH)\le d:= 6(d_{\Pi}+d_{\MV})\log(2eA(d_\Pi+d_{\MV})).$$ Now, since $W'_{P_k}(x,a)\ge \mu_k\ge \epsilon_k/A,$ we have $g(x,a,x')\le A/\epsilon_k$. Under the desired event of Lemma~\ref{lemma:find-distribution-succeed}, we have $\Var[g(x,a,x')]\le 110A$. Invoking Corollary~\ref{cor:bernstein_uniform_convergence}, by setting 
\begin{align*}
n_{k}^{\text{cb}}&=\frac{c_5A}{\epsilon_k^2}\left(d\ln\left(\frac{A}{\epsilon_k}\right)+\ln(140L^2C/\delta)\right),
\end{align*}
for some large enough constant $c_5$, we have, with probability at least $1-\delta/14$, for the first $LC$ times that $\Eliminate(g,h,j)$ is called
\begin{equation}\label{eq:lem-eliminate_eta_1}
\left|\frac{1}{n_k^{\text{cb}}}\sum_{p=1}^{n_k^{\text{cb}}}v(x_{h+1}^{p})\frac{\mathbb{I}[\pi(x_h^{p})=a_h^{p}]}{W'_{P_k}(x_h^p,a_h^p)}-\E_{x\sim \MD_{\pi',h}}\E_{x'\sim p(\cdot\mid x,\pi(x))}[v(x')]\right|<\epsilon_k/4,
\end{equation}
for all $(\pi,v)\in \Pi\times \MV, 1\le k\le j.$ Similarly, by applying Corollary~\ref{cor:bernstein_uniform_convergence} on hypothesis $g_{\pi,v}(x,a,r):= r\cdot \frac{\mathbb{I}[\pi(x)=a]}{W'_{P_k}(x,a)},$ we have with probability at least $1-\delta/14,$ for the first $LC$ times that $\Eliminate(g,h,j)$ is called
\begin{equation}\label{eq:lem-eliminate_eta_2}
\left|\frac{1}{n_k^{\text{cb}}}\sum_{p=1}^{n_k^{\text{cb}}}r_h^p\frac{\mathbb{I}[\pi(x_h^{p})=a_h^{p}]}{W'_{P_k}(x_h^p,a_h^p)}-\E_{x\sim \MD_{\pi',h}}[r(x,\pi(x))]\right|<\epsilon_k/4,
\end{equation}for all $(\pi,v)\in \Pi\times \MV, 1\le k\le j.$ 
Combining Eq.~\eqref{eq:lem-eliminate_eta_1} and Eq.~\eqref{eq:lem-eliminate_eta_2},  we get Lemma~\ref{lem:eliminate_eta}.
\end{proof}
Similarly, by setting $$n_{k}=\frac{c_6AM}{\epsilon_k^2}\left(d\ln\left(\frac{A}{\epsilon_k}\right)+\ln(140L^2C/\delta)\right)$$ for some large enough constant $c_6,$ Lemma~\ref{lem:eliminate_phi} holds for the infinite hypothesis setting.

\subsection{Regret analysis}
In this section we prove Theorem~\ref{thm:main-infinite}.
\begin{proof}[Proof sketch of Theorem~\ref{thm:main-infinite}]
The proof of Theorem~\ref{thm:main-infinite} is exactly the same as proof of Theorem~\ref{thm:main}, except for changing the value of parameters. It can be shown in the same way that,
\begin{itemize}
\item the regret incurred by one invocation of $\Identify(G,h,k)$ is bounded by
\[
O(H^2\ln^3(P)\ln(P/\delta)/\epsilon_k);
\]
\item the regret incurred by one invocation of $\Chk(G,h,j)$ is bounded by
\[
O(H^2\ln^3(P)\ln(P/\delta)/\epsilon_j);
\]
\item the regret incurred by one invocation of $\Eliminate(g,h,j)$ is bounded by
\[
c_{\rm ELIM} MAH^2(\ln^3(P)+d\ln(P))\ln(P/\delta)/\epsilon_j.
\]
\end{itemize}

Therefore, the overall regret our algorithm is bounded by,
\[
{O}\left(M^2AH^4 \ln(P)(\ln^3(P)+6(d_{\Pi}+d_{\MV})\ln(2eA(d_\Pi+d_{\MV}))\ln(P))\ln(P/\delta)/\epsilon+n\epsilon\right) .
\]
\end{proof}

\section{Probabilistic tools}
In this section we provide some probabilistic tools that are used in the proof.

The following lemma is a classical result of uniform convergence. 
\begin{lemma}[Theorem 29.1 of \cite{devroye2013probabilistic}, \cite{pollard2012convergence}]\label{lem:covering}
Let $\MG\subset [0,b]^{\MZ}$ be a function class and $\MD$ a distribution over $\MZ$, where $|g(z)|\le b,\forall g\in \MG,z\in \MZ$. Let $\{z_1,z_2,\cdots,z_n\}\sim \MD^n$ be $n$ i.i.d. samples. For any $n$ and $\epsilon>0$, 
\begin{equation}
\Pr\left\{\sup_{g\in \MG}\left|\frac{1}{n}\sum_{i=1}^{n}g(z_i)-\E_{z\sim \MD}[g(z)]\right|>\epsilon\right\}\le 8\MN(\epsilon/8,\MG,n)\exp\left(-\frac{n\epsilon^2}{128b^2}\right).
\end{equation}
\end{lemma}
Next lemma is an extension of the classical Bernstein inequality.
\begin{lemma}[Lemma 3.1 of \cite{massart1986rates}]\label{lem:Bernstein_without_replacement}
For any $N\ge 1$, let $w$ be an uniformly random permutation over $[N]$. For any $\xi\in \R^{N}$, define 
$$S_N=\sum_{i=1}^{n}\xi_i,\quad \tilde{S}_n=\sum_{i=1}^{n}\xi_{w(i)},\quad \sigma_N^2=\left(\frac{1}{N}\sum_{i=1}^{N}\xi_i^2\right)-\left(\frac{1}{N}\sum_{i=1}^{N}\xi_i\right)^2,$$ and $U_N=\max_{1\le i\le N}\xi_i-\min_{1\le i\le N}\xi_i.$ Then for any $\epsilon>0$,
\begin{equation}
\Pr\left\{\left|\frac{\tilde{S}_n}{n}-\frac{S_N}{N}\right|>\epsilon\right\}\le 2\exp\left(-\frac{n\epsilon^2}{2\sigma_N^2+\epsilon U_N}\right).
\end{equation}
\end{lemma}

The following lemma is an adaption of Theorem 3.3 in \cite{massart1986rates}.
\begin{lemma}[Bernstein version of Lemma~\ref{lem:covering}]\label{lem:bernstein_uniform_convergence}
Let $\MG\subset [0,b]^{\MZ}$ be a function class and $\MD$ a distribution over $\MZ$, where $\Var_{z\sim \MD}\left[g(z)\right]\le a^2,\forall g\in \MG$ and $|g(z)|\le b,\forall g\in \MG,z\in \MZ$. Then, for any $\epsilon > 0, $, 
\begin{align}
&\Pr_{\{z_1,\cdots,z_n\}\sim \MD^n}\left\{\sup_{g\in \MG}\left|\frac{1}{n}\sum_{i=1}^{n}g(z_i)-\E_{z\sim \MD}\left[g(z)\right]\right|>\epsilon\right\}\notag\\
\le ~&\inf_{N\ge 2n+8b^2/\epsilon^2} \left(16\MN(a^2/16b,\MG,N)
	\exp\left(-\frac{Na^4}{128b^4}\right)+4\MN\left(\frac{\epsilon n}{16N},\MG,N\right)\exp\left(-\frac{n\epsilon^2}{256a^2+64\epsilon b}\right)\right).
\end{align}
\end{lemma}
\begin{proof} The lemma is proved in three steps.
\paragraph{Step 1: Ghost sampling.} Let $n'=N-n$. Let $\{z_1,\cdots,z_N\}\sim \MD^{N}$ be $N$ i.i.d.~random samples. For any $g\in \MG$, we define the following shorthand, 
\[
\hat{\E}_{n}[g(z)] := \frac{1}{n}\sum_{i=1}^{n}g(z_i),\quad \hat{\E}_{n'}[g(z)]:= \frac{1}{n'}\sum_{i=n+1}^{N}g(z_i),\quad \hat{\E}_{N}[g(z)]:= \frac{1}{N}\sum_{i=1}^{N}g(z_i).
\]
In this step, we prove that for $N\ge 2n+8b^2/\epsilon^2$, 
\begin{equation}
\Pr\left\{\sup_{g\in \MG}\left|\hat{\E}_n[g(z)]-\E_{x\sim \MD}[g(z)]\right|>\epsilon\right\}\le 2\Pr\left\{\sup_{g\in \MG}\left|\hat{\E}_n[g(z)]-\hat{\E}_{n'}[g(z)]\right|>\epsilon/2\right\}.
\end{equation}

Let $Z_n=\{x_1,\cdots,x_n\}$ and $Z_{n'}=\{x_{n+1},\cdots,x_{N}\}$ Given $Z_n$, let $U(Z_{n})$ be the event that $$\sup_{g\in \MG}\left|\hat{\E}_n[g(x)]-\E_{x\sim \MD}[g(x)]\right|>\epsilon.$$ If $U(Z_n)$ occurs, then there exists $g^*\in \MG$ such that $\left|\hat{\E}_n[g^*(x)]-\E_{x\sim \MD}[g^*(x)]\right|>\epsilon.$ Then we have,
\begin{align}
&\Pr\left\{\sup_{g\in \MG}\left|\hat{\E}_n[g(z)]-\hat{\E}_{n'}[g(z)]\right|>\epsilon/2\right\}\\
\ge ~&\Pr\left\{U(Z_{n})\right\}\Pr\left\{\sup_{g\in \MG}\left|\hat{\E}_{n}[g(x)]-\hat{\E}_{n'}[g(x)]\right|>\epsilon/2 \mid U(Z_{n})\right\}\\
\ge ~&\Pr\left\{U(Z_{n})\right\}\Pr\left\{\left|\hat{\E}_{n}[g^*(x)]-\hat{\E}_{n'}[g^*(x)]\right|>\epsilon/2 \mid U(Z_{n})\right\}\\
\ge ~&\Pr\left\{U(Z_{n})\right\}\Pr\left\{\left|\hat{\E}_{n}[g^*(x)]-\E_{x\sim \MD}[g^*(x)]\right|>\epsilon , \left|\hat{\E}_{n'}[g^*(x)]-\E_{x\sim \MD}[g^*(x)]\right|\le \epsilon/2 \mid U(Z_{n})\right\}\\
\ge ~&\frac{1}{2}\Pr\left\{U(Z_{n})\right\}\Pr\left\{\left|\hat{\E}_{n}[g^*(x)]-\E_{x\sim \MD}[g^*(x)]\right|>\epsilon \mid U(Z_{n})\right\}
\label{equ:step1_hoeffding}\\
\ge ~&\frac{1}{2}\Pr\left\{U(Z_{n})\right\},
\end{align}
where Eq.~\eqref{equ:step1_hoeffding} comes from Hoeffding inequality for $n'=N-n\ge 8b^2/\epsilon^2.$

\paragraph{Step 2: Symmetrization.} Let $w$ be a random permutation over $[N],$ independent of the choice of $\{z_1,\cdots,z_{N}\}.$ Define $$\hat{\E}_{w,n}[g(z)]:= \frac{1}{n}\sum_{i=1}^{n}g(z_{w(i)}),\quad \hat{\E}_{w,n'}[g(z)]:= \frac{1}{n'}\sum_{i=n+1}^{N}g(z_{w(i)})$$ In this step we prove that 
\begin{equation}
\Pr\left\{\sup_{g\in \MG}\left|\hat{\E}_n[g(z)]-\hat{\E}_{n'}[g(z)]\right|>\epsilon/2\right\}\le \Pr\left\{\sup_{g\in \MG}\left|\hat{\E}_{w,n}[g(z)]-\hat{\E}_{N}[g(z)]\right|>\epsilon/4\right\}.
\end{equation}

Note that since $\{z_1,z_2,\cdots,z_N\}$ has the same distribution as $\{z_{w(1)},z_{w(2)},\cdots,z_{w(N)}\},$ we have
$$\Pr\left\{\sup_{g\in \MG}\left|\hat{\E}_n[g(z)]-\hat{\E}_{n'}[g(z)]\right|>\epsilon/2\right\}=\Pr\left\{\sup_{g\in \MG}\left|\hat{\E}_{w,n}[g(z)]-\hat{\E}_{w,n'}[g(z)]\right|>\epsilon/2\right\}.$$

It follows from basic algebra that, 
\begin{align*}
&\frac{n'}{N}\left|\hat{\E}_{n,w}[g(z)]-\hat{\E}_{n',w}[g(z)]\right|
=\left|\frac{n'}{N}\hat{\E}_{n,w}[g(z)]-\frac{1}{N}\sum_{i=n+1}^{N}[g(z_{w(i)})]\right|\\
& \qquad\qquad =\left|\left(\frac{n'}{N}+\frac{n}{N}\right)\hat{\E}_{n,w}[g(z)]-\frac{1}{N}\sum_{i=1}^{N}[g(z_{w(i)})]\right|
=\left|\hat{\E}_{n,w}[g(z)]-\hat{\E}_{N}[g(z)]\right|.
\end{align*}
When $N>2n$ we get, $$\frac{1}{2}\left|\hat{\E}_{n,w}[g(z)]-\hat{\E}_{n',w}[g(z)]\right|\le \left|\hat{\E}_{n,w}[g(z)]-\hat{\E}_{N}[g(z)]\right|.$$ Therefore,
$$\Pr\left\{\sup_{g\in \MG}\left|\hat{\E}_{w,n}[g(z)]-\hat{\E}_{w,n'}[g(z)]\right|>\epsilon/2\right\}\le \Pr\left\{\sup_{g\in \MG}\left|\hat{\E}_{w,n}[g(z)]-\hat{\E}_{N}[g(z)]\right|>\epsilon/4\right\}.$$

\paragraph{Step 3: Covering Argument.} Define $$\hat{\Var}_{N}[g(z)]:= \frac{1}{N}\sum_{i=1}^{N}g(z_i)^2-\left(\frac{1}{N}\sum_{i=1}^{N}g(z_i)\right)^2.$$ First we show that,
\begin{equation}\label{equ:step3_part1}
\Pr\left\{\sup_{g\in \MG}\hat{\Var}_N[g(z)]>2a^2\right\}\le 8\MN\left(\frac{a^2}{16b},\MG,N\right)\exp\left(-\frac{Na^4}{128b^4}\right).
\end{equation}
Let $\MG^2=\{g^2\mid g\in \MG\}.$ Invoking Lemma~\ref{lem:covering}, we have
$$\Pr\left\{\sup_{g\in \MG}\left|\hat{\Var}_N[g(z)]-\Var[g(z)]\right|>a^2\right\}\le 8\MN\left(a^2/8,\MG^2,N\right)\exp\left(-\frac{Na^4}{128b^4}\right).$$

Note that, since $\left|g_1(z)^2-g_2(z)^2\right|=\left|(g_1(z)-g_2(z))(g_1(z)+g_2(z))\right|\le 2b\left|g_1(z)-g_2(z)\right|$ for all $g\in \MG,z\in \MZ$, we have $\MN\left(a^2/8,\MG^2,N\right)\le \MN\left(a^2/16b,\MG,N\right).$ Combining with the fact that $\Var[g(z)]\le a^2$, we get Eq.~\eqref{equ:step3_part1}.

Let $Z_N=\{z_1,\cdots,z_N\}.$ Let $\MC=\{g'_1,\cdots,g'_{|\MC|}\}$ be the minimal $(n\epsilon/16N)$-cover over $\MG\rvert_{Z_N}$. We have $|\MC|\le \MN(n\epsilon/16N,\MG,N),$ and there exists a function $\pi:\MG\to [|\MC|]$ such that, $$\frac{1}{N}\left|\sum_{i=1}^{N}\left(g(z_i)-g'_{\pi(g)}(z_i)\right)\right|\le n\epsilon/16N,\quad\forall g\in \MG.$$ Note that $$\left|\hat{\E}_{w,n}[g(z)]-\hat{\E}_{w,n}[g'_{\pi(g)}(z)]\right|= \frac{1}{n}\left|\sum_{i=1}^{n}\left(g(z_{w(i)})-g'_{\pi(g)}(z_{w(i)})\right)\right|\le \frac{N}{n}\cdot n\epsilon/16N=\epsilon/16.$$ Consequently, 
\begin{align*}
&\sup_{g\in \MG}\left|\hat{\E}_{w,n}[g(z)]-\hat{\E}_{N}[g(z)]\right|\\
\le~ &\sup_{g\in \MG}\left|\hat{\E}_{w,n}[g'_{\pi(g)}(z)]-\hat{\E}_{N}[g'_{\pi(g)}(z)]\right|+\sup_{g\in \MG}\left|\hat{\E}_{w,n}[g(z)]-\hat{\E}_{w,n}[g'_{\pi(g)}(z)]\right|+\sup_{g\in \MG}\left|\hat{\E}_{N}[g(z)]-\hat{\E}_{N}[g'_{\pi(g)}(z)]\right|\\
\le~ &\sup_{g'\in \MC}\left|\hat{\E}_{w,n}[g'(z)]-\hat{\E}_{N}[g'(z)]\right|+\epsilon/16+\epsilon n/16N\\
\le~ &\sup_{g'\in \MC}\left|\hat{\E}_{w,n}[g'(z)]-\hat{\E}_{N}[g'(z)]\right|+\epsilon/8.
\end{align*}
Therefore for any $Z_n\in \MZ^{n}$, $$\Pr\left\{\sup_{g\in \MG}\left|\hat{\E}_{w,n}[g(z)]-\hat{\E}_{N}[g(z)]\right|>\epsilon/4\mid Z_N\right\}\le \Pr\left\{\sup_{g'\in \MC}\left|\hat{\E}_{w,n}[g'(x)]-\hat{\E}_N[g'(x)]\right|>\epsilon/8\mid Z_N\right\}.$$
Given $Z_N$, let $\sigma_{N}^2=\sup_{g'\in \MC}\hat{\Var}_N[g'(z)]$ By Lemma~\ref{lem:Bernstein_without_replacement} and union bound, we get
\begin{equation}
\Pr\left\{\sup_{g'\in \MC}\left|\hat{\E}_{w,n}[g'(x)]-\hat{\E}_N[g'(x)]\right|>\epsilon/8\mid Z_N\right\}\le 2|\MC|\exp\left(-\frac{n\epsilon^2/64}{2\sigma_N^2+\epsilon b}\right).
\end{equation}
Combining with Eq.~\eqref{equ:step3_part1}, we have
\begin{align*}
&\Pr\left\{\sup_{g\in \MG}\left|\hat{\E}_{w,n}[g(x)]-\hat{\E}_N[g(x)]\right|>\epsilon/4\right\}\\
& \qquad\qquad \le 2\MN(n\epsilon/16N,\MG,N)\exp\left(-\frac{n\epsilon^2/64}{4a^2+\epsilon b}\right)+8\MN\left(\frac{a^2}{16b},\MG,N\right)\exp\left(-\frac{Na^4}{128b^4}\right).
\end{align*}

The result follows from combining the three steps together.
\end{proof}
\begin{corollary}\label{cor:bernstein_uniform_convergence}
By setting $N=3nb^4/a^4,$ we get
\begin{equation}\label{equ:bernstein_cor}
\Pr_{\{z_1,\cdots,z_n\}\sim \MD^n}\left\{\sup_{g\in \MG}\left|\frac{1}{n}\sum_{i=1}^{n}g(z_i)-\E_{z\sim \MD}\left[g(z)\right]\right|>\epsilon\right\}\le 20\MN\left(\frac{\epsilon a^4}{48b^4},\MG,\frac{3nb^4}{a^4}\right)\exp\left(-\frac{n\epsilon^2/64}{4a^2+\epsilon b}\right).
\end{equation}
\end{corollary}
\begin{proof} Note that since $b^2\ge a^2,$ the inequality holds trivially when $n\le 8a^4/(\epsilon^2 b^2).$ When $n>8a^4/(\epsilon^2 b^2),$ we have $N=3nb^4/a^4>2n+8b^2/\epsilon^2.$ Eq.~\eqref{equ:bernstein_cor} follows by monotonicity of $\MN(\cdot,\MG,N)$ and $\exp(\cdot).$
\end{proof}

\end{document}